\documentclass[12pt]{article}
\usepackage{authblk}

\usepackage[utf8]{inputenc}
\usepackage[margin=1in]{geometry}
\usepackage{amsmath}
\usepackage{algorithm}
\usepackage[noend]{algpseudocode}
\usepackage{amsfonts}
\usepackage{amssymb}
\usepackage{bm}
\usepackage[numbers]{natbib}
\usepackage{bbm}
\usepackage{textcomp}
\usepackage{pgfplots}
\usepackage{bbm}
\usepackage{mathabx}
\usepackage{graphicx}
\usepackage{wrapfig}
\usepackage[]{youngtab}
\usepackage{amsthm}
\usepackage{tikz}
\usepackage{tikz-cd}
\usetikzlibrary{automata, positioning}
\usepackage{comment}
\usepackage{stmaryrd}

\usepackage{theoremref}
\usepackage{mathrsfs}
\usepackage{color-edits}
\addauthor{sw}{blue}
\addauthor{jw}{teal}
\addauthor{bw}{purple}

\usepackage[english]{babel}
\usepackage[utf8]{inputenc}
\usepackage{fancyhdr}
\pgfplotsset{width=10cm,compat=1.9}

\addauthor{vs}{red}
\addauthor{cj}{orange}

\renewcommand{\P}{\mathbb{P}}
\newcommand{\N}{\mathbb{N}}
\newcommand{\E}{\mathbb{E}}
\newcommand{\R}{\mathbb{R}}

\newcommand{\calF}{\mathcal{F}}
\newcommand{\calX}{\mathcal{X}}

\newcommand{\calY}{\mathcal{Y}}

\newcommand{\calG}{\mathcal{G}}

\newcommand{\calI}{\mathcal{I}}

\newcommand{\calZ}{\mathcal{Z}}
\newcommand{\calA}{\mathcal{A}}
\newcommand{\calN}{\mathcal{N}}

\newcommand{\calW}{\mathcal{W}}
\newcommand{\calD}{\mathcal{D}}

\newcommand{\calU}{\mathcal{U}}

\newcommand{\Cal}
{\mathrm{Cal}}
\newcommand{\range}{\mathrm{range}}

\newcommand{\wh}[1]{\widehat{#1}}
\newcommand{\wt}[1]{\widetilde{#1}}
\newcommand{\wb}[1]{\overline{#1}}

\newcommand{\err}{\mathrm{err}}
\newcommand{\corr}{\mathrm{Corr}}
\newcommand{\ATE}{\mathrm{CATE}}
\newcommand{\QTE}{\mathrm{QUT}}
\newcommand{\LATE}{\mathrm{LATE}}
\newcommand{\ACD}{\mathrm{ACD}}
\newcommand{\rate}{\mathrm{rate}}

\algrenewcommand\algorithmicrequire{\textbf{Input:}}
\algrenewcommand\algorithmicensure{\textbf{Output:}}

\usepackage{thm-restate}
\newtheorem{theorem}{Theorem}
\numberwithin{theorem}{section}
\newtheorem{lemma}[theorem]{Lemma}
\newtheorem{informal}{Informal Theorem}

\newtheorem{prop}[theorem]{Proposition}
\newtheorem{corollary}[theorem]{Corollary}

\newtheorem{ass}{Assumption}

\theoremstyle{definition}
\newtheorem{definition}[theorem]{Definition}
\newtheorem{example}[theorem]{Example}

\theoremstyle{remark}

\usepackage{hyperref}
\hypersetup{
	colorlinks=false,
	bookmarks=true,
	breaklinks=true,
	hidelinks=true,
	pdfpagemode=empty,
}
\usepackage{nameref}
\usepackage{subcaption}
\usepackage{authblk}
\usepackage[noabbrev, capitalize, nameinlink]{cleveref}

\newcommand{\cj}[1]{\textcolor{orange}{[CJ: #1]}}

\crefname{prop}{Proposition}{Propositions}
\crefname{rmk}{Remark}{Remarks}
\crefname{cor}{Corollary}{Corollaries}
\crefname{claim}{Claim}{Claims}
\crefname{lemma}{Lemma}{Lemmata}
\crefname{example}{Example}{Examples}
\crefname{corollary}{Corollary}{Corollaries}

\title{Orthogonal Causal Calibration}
\author[1]{Justin Whitehouse}
\author[3]{Christopher Jung} 
\author[1]{Vasilis Syrgkanis} 
\author[2]{Bryan Wilder}
\author[2]{Zhiwei Steven Wu}
\affil[1]{Stanford University}
\affil[2]{Carnegie Mellon University}
\affil[3]{Meta Research}

\date{\today}

\begin{document}
\maketitle
\begin{abstract}

    Estimates of heterogeneous treatment effects such as conditional average treatment effects (CATEs) and conditional quantile treatment effects (CQTEs) play an important role in real-world decision making. Given this importance, one should ensure these estimates are calibrated. 
    While there is a rich literature on calibrating estimators of non-causal parameters, very few methods have been derived for calibrating estimators of causal parameters, or more generally estimators of quantities involving nuisance parameters. 

    In this work, we develop general algorithms for reducing the task of causal calibration to that of calibrating a standard (non-causal) predictive model.
    Throughout, we study a notion of calibration defined with respect to an arbitrary, nuisance-dependent loss $\ell$, under which we say an estimator $\theta$ is calibrated if its predictions cannot be changed on any level set to decrease loss.
    For losses $\ell$ satisfying a condition called \textit{universal orthogonality}, we present a simple algorithm that transforms partially-observed data into generalized pseudo-outcomes and applies any off-the-shelf calibration procedure.
    For losses $\ell$ satisfying a weaker assumption called \textit{conditional orthogonality}, we provide a similar sample splitting algorithm the performs empirical risk minimization over an appropriately defined class of functions. 
    Convergence of both algorithms follows from a generic, two term upper bound of the calibration error of any model: one term that measures the error in estimating unknown nuisance parameters and another that measures calibration error in a hypothetical world where the learned nuisances are true. We demonstrate the practical applicability of our results in experiments on both observational and synthetic data.
    Our results are exceedingly general, showing that essentially any existing calibration algorithm can be used in causal settings, with additional loss only arising from errors in nuisance estimation.


\end{abstract}

\thispagestyle{empty} \setcounter{page}{0}
\clearpage
\tableofcontents
\thispagestyle{empty} \setcounter{page}{0}
 \clearpage

\section{Introduction}
\label{sec:intro}

Estimates of heterogeneous causal effects such as conditional average treatment effects (CATEs), conditional average causal derivatives (CACDs), and conditional quantile treatment effects (CQTEs) play a pervasive role in understanding various statistical, scientific, and economic problems. 
Due to the partially-observed nature of causal data, estimating causal quantities is naturally more difficult than estimating non-causal ones. Despite this, a vast literature has developed focused on estimating causal effects using both classical statistical methods \citep{kennedy2022semiparametric, kennedy2023semiparametric, nie2021quasi, abrevaya2015estimating,froolich2010estimation} and modern machine learning (ML) approaches \citep{foster2023orthogonal, chernozhukov2022riesznet, chernozhukov2018double, semenova2021debiased, fan2022estimation,heckman2005structural}. In this work, we focus on an complementary and relatively understudied direction --- developing algorithms for \textit{calibrating} estimates of heterogeneous causal effects.

Calibration is a notion of model consistency that has been extensively studied both in theory and practice by the general ML community~\citep{guo2017calibration, minderer2021revisiting, lakshminarayanan2017simple, malinin2018predictive, thulasidasan2019mixup, hendrycks2019augmix, ovadia2019can, hebert2018multicalibration, dawid1985calibration, dawid1982well, foster1998asymptotic}. Classically, a model $\wh{\theta}(X)$ predicting $Y$ is said to be calibrated if, when it makes a given prediction, the observed outcome on average is said prediction, i.e.\ if
\[
\E[Y \mid \wh{\theta}(X)=v]= v \qquad \text{for any } v \in \mathrm{range}(\wh{\theta}).
\]
There are deep connections between calibrated predictions and optimal downstream decision-making~\citep {noarov2023high, noarov2023scope}. Namely, for many utility functions, when a decision maker is equipped with calibrated predictions, the optimal mapping from prediction to decision simply treats those predictions as if they were the ground truth. 
Given that models predicting heterogeneous causal effects or counterfactual quantities are leveraged in decision-making in domains such as medicine~\citep{feng2012generalized, kent2018personalized}, advertising and e-commerce~\citep{hitsch2024heterogeneous, gao2024causal}, and public policy~\citep{wilder2024learning}, it is essential that these estimates are calibrated. In non-causal settings, there exists a plethora of simple algorithms for performing calibration such of algorithms such as histogram binning~\citep{gupta2021distribution}, isotonic calibration~\citep{barlow1972isotonic, van2023causal}
, linear calibration, and Platt scaling~\citep{platt1999probabilistic, niculescu2005predicting}.
These algorithms are all univariate regression procedures, regressing an observed outcome $Y$ onto a model prediction $\wh{\theta}(X)$ over an appropriately-defined function class. However, given the partially-observed nature of causal problems, it is non-obvious how to extend these algorithms for calibrating estimates of general causal effects.

To illustrate the importance and difficulty of calibrating causal models, we look at the example of a doctor using model predictions to aid in prescribing medicine to prevent heart attacks. The doctor may have access to observations of the form $Z = (X, A, Y)$, where $X \in \calX$ represent covariates, $A \in \{0, 1\}$ is a binary treatment, and $Y = AY(1) + (1 - A)Y(0) \in \R$ represents an observed outcome. For instance, a doctor may use a patient's age, height, and blood pressure (covariates $X$) to decide whether or not to give medicine (treatment $A$) to a patient in order to prevent a heart attack (outcome $Y$). To aid in their decision making, the doctor may want to employ a model $\theta(X)$ predicting the \textit{conditional average treatment effect (CATE)}, which is defined as
\[
\theta_{\ATE}(x) := \E[Y(1) - Y(0) \mid X = x]
\]
In this example, the CATE measures the difference in probability of a heart attack under treatment and control. Naturally, a doctor may decide to prescribe medication to a patient if $\theta(X) < 0$, i.e.\ if the model predicts the probability of a heart attack will go down under treatment. If the model is miscalibrated---for instance if we have
\[
\E[Y(1) - Y(0) \mid \theta(X) = -0.15] > 0
\]
the doctor may actually harm patients by prescribing medication! If $\theta(X)$ were calibrated, such a risk would not a occur. However, the doctor cannot run one of the aforementioned calibration algorithms because they only ever observe the outcome under treatment $Y(1)$ or control $Y(0)$, never the target individual treatment effect $Y(1) - Y(0)$.

To work around this, one typically needs to estimate \textit{nuisance parameters} --- functions of the underlying data-generating distribution that are used to ``de-bias'' partial observations (either $Y(1)$ or $Y(0)$ in the previous example) into \textit{pseudo-outcomes}: quantities that, on average, look like the target treatment effect (here $Y(1) - Y(0)$). For the CATE, these nuisance functions are the \textit{propensity score} $\pi_0(x) := \P(A = 1 \mid X = x)$ and the \textit{expected outcome mapping} $\mu_0(a, x) := \E[Y \mid X = x, A = a]$, and pseudo-outcomes are of the form~\citep{kennedy2023towards}
\[
\chi_{\ATE}((\mu, \pi); W) := \mu(1, X) - \mu(0, X) + \left(\frac{A}{\pi(X)} - \frac{1 - A}{1 - \pi(X)}\right)(Y - \mu(A, X)).
\]
Naturally, the nuisances and pseudo-outcomes will naturally be different for other heterogeneous causal effects and can typically be derived by the statistician.

What is known about causal calibration? In the setting of the above CATE example, if we assume \textit{conditional ignorability}\footnote{This is the condition that the potential outcomes $Y(0)$ and $Y(1)$ are conditionally independent of the treatment $A$ given covariates $X$, i.e.\ $Y(0), Y(1) \perp A \mid X$}, \citet{van2023causal} show that by estimating the propensity score and expected outcome mappings, one can perform isotonic regression of the doubly-robust pseudo-outcomes $\chi_{\ATE}$ onto model predictions $\theta(X)$ to ensure asymptotic calibration.
However, this result is limited in application, and doesn't extend to other heterogeneous effects of interest like conditional average causal derivatives, conditional local average treatment effects, or even CATEs in the presence of unobserved confounding (in which case one may need instrumental variables).
Furthermore, the result of \citet{van2023causal} don't allow a learner to calibrate using other algorithms like Platt scaling, histogram binning, or even simple linear regression.
While one could derive individual results for calibrating an estimate of ``parameter A'' under ``algorithm B'', this would likely lead to a repeated reinventing of the wheel. 
The question considered in this paper is thus as follows: can one construct a framework that allows a statistician to calibrate estimates of \textit{general heterogeneous causal effects} using an \textit{arbitrary, off-the-shelf calibration algorithms}?

\subsection{Our Contributions}

In this paper, we reduce the problem of \textit{causal calibration}, or calibrating estimates of heterogeneous causal effects, to the well-studied problem of calibrating non-causal predictive models. We assume the scientist is interested in calibrating an estimate $\theta(X)$ of some heterogeneous causal effect $\theta_0(X)$ that can be specified as the conditional minimizer of a loss function, i.e.\ 
$\theta_0(x) \in \arg\min_{\nu}\E[\ell(\nu, g_0; W) \mid X = x]$.
Here, $\ell(\nu, g; w)$ involves a nuisance component, and we assume that there exists a true, unknown nuisance parameter $g_0$. We say a $\theta$ is perfectly calibrated with respect to $\ell$ if $\E[\partial \ell(\theta(X), g_0; Z) \mid \theta(X)] = 0$,\footnote{Here, $\partial \ell(\theta(x), g; z)$ denotes the partial derivative of $\ell$ with respect to its first argument $\theta(x)$, i.e.\ the quantity is defined by $\frac{\partial}{\partial \nu}\ell(\nu, g; z) \vert_{\nu = \theta(x)}$.} and that $\theta$ is approximately calibrated if the $L^2$ error
\[
\Cal(\theta, g) = \E\left(\E[\partial\ell(\theta(X), g; Z) \mid \theta(X)]^2\right)^{1/2}
\]
is small when $g$ is set as the true nuisance $g_0$. In words, $\theta$ is calibrated if the predictions of $\theta$ are ``unimprovable'' with respect to current level sets, a condition similar to the one outlined by \citet{noarov2023scope} and the concept of swap regret \citep{blum2007external, foster1998asymptotic}.

As a concrete example, $\theta_{\ATE}(X)$ is the conditional minimizer of the \textit{doubly-robust loss} $\ell_{\ATE}(\nu, (\mu, \pi); z)$ given by $\ell_{\ATE}(\nu, (\mu, \pi); z) := \frac{1}{2}\Big(\nu - \chi_{\ATE}((\mu, \pi); z)\Big)^2$,
where $\pi_0$, $\mu_0$, and $\chi_{\ATE}$ are as defined above. Perfect calibration for a CATE estimate $\theta$ becomes $\E[Y(1) - Y(0) \mid \theta(X)] = \theta$, and approximate $L^2$ calibration becomes $\Cal(\theta, (\mu_0, \pi_0)) := \E\left(\left(\theta(X) - \E[Y(1) - Y(0) \mid \theta(X)]\right)^2\right) \approx 0$.

Our reduction and algorithms are based around a robustness condition on the underlying loss $\ell$ called \textit{Neyman Orthogonality}. We say a loss $\ell$ is Neyman orthogonal if
\begin{equation}
\label{eq:neyman_ref}   
D_g\E[\partial \ell(\theta_0(X), g_0; Z) \mid X](g - g_0) = 0,
\end{equation}
i.e.\ if $\ell$ is insensitive to small estimation errors in the nuisance parameters $g_0$ and loss minimizing parameter $\theta_0$.\footnote{Here, $D_g$ denotes a Gateaux derivative.}
In our work, we consider two mild variants of Neyman orthogonality. The first is \textit{universal orthogonality} \citep{foster2023orthogonal}, a stronger notion of orthogonality in which $\theta_0(x)$ in Equation~\eqref{eq:neyman_ref} can be replaced by \textit{any} estimate $\theta(x)$. Examples of causal effects that minimize such a loss are CATEs, conditional average causal derivatives (CACD), and conditional local average treatment effects (CLATEs). Another condition we consider is called \textit{conditional orthogonality}, in which instead of conditioning of covariates $X$ in Equation~\eqref{eq:neyman_ref}, one conditions on a post-processing $\varphi(X)$ of covariates instead.  Conditional orthogonality is a natural condition for calibration tasks, as we ultimately care about assessing the quality of an estimator $\theta(X)$ conditional on its own predictions. An example causal parameter that can be specified as the minimizer of a conditionally orthogonal loss is the conditional quantile under treatment (CQUT). Our specific contributions, broken down by assumption on the loss, are as follows:
\begin{enumerate}
\item In Section~\ref{sec:universal}, we study calibration with respect to \textit{universally orthogonal} loss functions.   
We present a sample splitting algorithm (Algorithm~\ref{alg:sample-split-universal}) that uses black-box ML algorithms to estimate unknown nuisance functions on the first half of the data, transforms the second half into de-biased ``pseudo-outcomes'', and then applies any off-the-shelf calibration algorithm on the second half of the data treating pseudo-outcomes as true labels. We provide high-probability, finite-sample guarantees for our algorithm that only depend on convergence rates for the black-box nuisance estimation and calibration algorithms. We additionally provide a cross-calibration algorithm (Algorithm~\ref{alg:cross-cal-universal}) that makes more efficient use of the data and show that our calibration algorithms do not significantly increase the risk of a model (Appendix~\ref{app:loss_red}).
\item In Section~\ref{sec:conditional}, we study the calibration of models predicting effects that minimize a \textit{conditionally orthogonal} loss function.
In this more general setting, we describe a similar sample splitting algorithm that uses half of the data lo learn nuisances and then runs an off-the-shelf algorithm for ``generalized' calibration on the second half the data. We provide examples of such calibration algorithms in Appendix~\ref{app:gen_cal}, which are motivated by extending algorithms like isotonic calibration and linear calibration to general losses. We also prove finite-sample guarantees for this algorithm and construct a similar cross-calibration algorithm (Algorithm~\ref{alg:cross-cal-cond}).
\item Lastly Section~\ref{sec:experiments}, we empirically evaluate the performance of the aforementioned algorithms on a mix of observational and semi-synthetic data. In one experiment, we use observational data alongside Algorithm~\ref{alg:cross-cal-universal} to show how one can calibrate models predicting the CATE of 401(k) eligibility and CLATE of 401(k) participation on an individual's net total financial assets. Likewise, in another experiment on synthetically generated data, we show how Algorithm~\ref{alg:cross-cal-cond} can be used to both decrease $L^2$ calibration error and average loss for models predicting CQUTs.
\end{enumerate}

The design of our algorithms is inspired by a generic upper bound on the calibration error of any estimator $\theta(X)$. We show that the  $L^2$ calibration error of any estimator can be bounded above by two, decoupled terms: one involving nuisance estimation error and another representing calibration error under the orthogonalized loss evaluated at the learned nuisances.
\begin{informal}
    \label{thm:inf}
    Suppose $\theta : \calX \rightarrow \R$ is some estimator, $\ell$ is some base loss, and $\wt{\ell}$ is the corresponding \textit{orthogonalized loss}. Let $g_0$ denote the true, unknown nuisance functions, and $g$ arbitrary nuisance estimates. We have
    \[
    \underbrace{\Cal(\theta, g_0)}_{\text{$L^2$ calibration error under $\ell(\theta(x), g_0; w)$}} \lesssim \underbrace{\err(g, g_0)}_{\text{nuisance estimation error}} \qquad\quad + \underbrace{\Cal(\theta, g)}_{\text{$L^2$ calibration error under $\ell(\theta(x), g; w)$}}
    \]
\end{informal}
\noindent We view Informal Theorem~\ref{thm:inf} as a ``change of measure'' or ``change of nuisance'' result, allowing the learner to essentially pretend our learned nuisances represent reality while only paying a small error for misestimation. In many settings where there are two nuisances functions (e.g.\ $g = (\eta, \zeta)$ and $g_0 = (\eta_0, \zeta_0)$), we will have $\err(g, g_0) = O(\|(\eta - \eta_0)\cdot(\zeta - \zeta_0)\|)$, and thus the error will be small if we estimate at least one nuisance function sufficiently well. This in particular is the case for the aforementioned CATE, where $g_0 = (\mu_0, \pi_0)$, as noted above. More broadly, we will have $\err(g, g_0) = O(\|g - g_0\|^2)$. While simple in appearance, the above bound depends on deep connections between Neyman orthogonality and calibration error. We prove a bound of the above form in each of the main sections of the paper. This decoupled bound naturally suggest using some fraction of the data to learn nuisance parameters, thus minimizing the first term, and then using fresh data with the learned nuisances to perform calibration, thus minimizing the second term.

Given that our framework and results are quite general, we go through concrete examples to help with building intuition. 
In fact, we show that the work of \citet{van2023causal} can be seen as a special instantiation of our framework.

\subsection{Related Work}
\paragraph{Calibration:} Important to our work is the vast literature on calibration. Calibration was considered first in the context producing calibrated probabilities, both in the online \citep{dawid1982well, foster1998asymptotic} and i.i.d.\ \citep{platt1999probabilistic, zadrozny2001obtaining} settings, but has since been considered in other contexts such as distribution calibration \citep{song2019distribution}, threshold calibration~\citep{sahoo2021reliable, vovk2002line}, and parity calibration~\citep{chung2023parity}. Calibration is typically orthogonal to model training, and usually occurs as a simple post-processing routine. Some well-known algorithms for post-hoc calibration include Platt scaling \citep{platt1999probabilistic, gupta2023online}, histogram binning \citep{zadrozny2001obtaining, gupta2021distribution}, and isotonic regression \citep{zadrozny2002transforming, barlow1972isotonic}. 
Many of these algorithms simultaneously offer strong theoretical guarantees (see \citet{gupta2022post} for an overview) and strong empirical performance when applied to practically-relevant ML models \citep{guo2017calibration}. 
We view our work as complementary to existing, non-causal results on calibration. Our two-step algorithm allows a practitioner to directly apply any of the above listed algorithms, inheriting existing error guarantees so long as nuisance estimation is efficiently performed.


\paragraph{Double/debiased Machine Learning:} In our work, we also draw heavily from the literature on double/debiased machine learning \citep{chernozhukov2018double, chernozhukov2022debiased, chernozhukov2022automatic}. Methods relating to double machine learning aim to eschew classical non-parametric assumptions (e.g.\ Donsker properties) on nuisance functions, often through simple sample splitting schemes~\citep{hasminskii1979nonparametric, kennedy2022semiparametric, bickel1993efficient}. In particular, if target population parameters are estimated using a Neyman orthogonal loss function \citep{neyman1979calpha, foster2023orthogonal}, then these works show that empirical estimates of the population parameters converge rapidly to either a population or conditional loss minimizer.

Of the various works related to double/debiased machine learning, we draw most heavily on ideas from the framework of orthogonal statistical learning \citep{foster2023orthogonal}. In their work, \citet{foster2023orthogonal} develop a simple two-step framework for statistical learning in the presence of nuisance estimation. In particular, they show that when the underlying loss is Neyman orthogonal, then the excess risk can be bounded by two decoupled error terms: error from estimating nuisances and the error incurring from applying a learning algorithm with a fixed nuisance. 
Following its introduction, the orthogonal statistical learning framework has found applications in tasks such as the design of causal random forests \citep{oprescu2019orthogonal} and causal model ensembling via Q-aggregation \citep{lan2023causal}. In this work, we show that central ideas from orthogonal statistical learning are naturally applicable to the problem of calibrating estimators of causal parameters.  

Lastly, our work can be seen as a significant generalization of existing results on the calibration of causal parameters. Primarily, we compare our results to the work of \citet{van2023causal}. In their work, the authors construct a sample-splitting scheme for calibrating estimates of conditional average treatment effects (CATEs). The specific algorithm leveraged by the authors uses one half of the data to estimate nuisance parameters, namely propensities and expected outcomes under control/treatment. After nuisances are learned, the algorithm transforms the second half of the data into pseudo-observations and runs isotonic regression as a calibration procedure. Our results are applicable to estimates of any causal parameter that can be specified as the population minimizer of a loss function, not just CATEs. Additionally, our generic procedure allows the scientist to plug in any black-box method for calibration, not just a specific algorithm such as isotonic regression. Likewise, our work is also significantly more general than the work of \citet{leng2024calibration}, who provide a maximum-likelihood based approach for performing \textit{linear calibration}, a weaker notion of calibration, of CATE estimates. 

\section{Calibration of Causal Effects}
\label{sec:calibration}

We are interested in calibrating some estimator/ML model $\theta \in \Theta \subset \{f : \calX \rightarrow \R\}$ that is predicting some heterogeneous causal effect $\theta_0(x)$. We assume $\theta_0$ is specified as the conditional minimizer of some loss $\ell(\theta(x), g; z)$, i.e.\
\[
\theta_0(x) \in \arg\min_{\nu \in \R}\E[\ell(\nu, g_0; Z) \mid X = x].
\]
We assume $\ell: \R \times \calG \times \calZ \rightarrow \R$ is some generic loss function involving nuisance.  We assume $\calZ$ is some space containing observations, and write $Z$ as a prototypical random element from this space, and $P_Z$ as the distribution on $\calW$ from which $W$ is drawn. In the above, $\calG$ is some space of nuisance functions, which are of the form $g : \calW \rightarrow \R^d$. We assume that there is some true nuisance parameter $g_0 \in \calG$, but that this parameter is unknown to the learner and must be estimated.  We generally assume $\calG$ is a subset of $L^2(P_W) := L^2(\calW, P_W)$, and so as a norm we can consider the $\|g\|_{L^2(P_W)} := \int_{\calW} \|g(w)\|^2 P_W(dw)$, where $\|\cdot\|$ denotes the standard Euclidean norm on $\R^d$. Given a function $T : \calG \rightarrow \R$ and functions $f, g \in \calG$, we let $D_g T(f)(g)$ and $D_g^2 T(f)(g, g)$ denote respectively the first and second Gateaux derivatives of $T$ at $f$ in the direction $g$.

We typically have $Z = (X, A, Y)$, where $X \in \calX$ represents covariates, $A \in \calA$ represents treatment, and $Y \in \R$ represents an outcome in an experiment. More generally, we assume the nested structure $\calX \subset \calW \subset \calZ$, where $\calX$ intuitively represents the space of covariates or features, $\calW$ represents an extended set of parameters on which the true nuisance parameter may also depend (e.g.\ instrumental variables, whether or not an individual actually accepted a treatment), and $\calZ$ may contain additional observable information (e.g.\ outcome $Y$ under the given treatment). We write the marginal distributions of $X$ and $W$ respectively as $P_X$ and $P_W$. We typically write $\ell(\theta, g; w)$ instead of $\ell(\theta(x), g; w)$ for succinctness, and we let $\partial \ell(\theta, g; w) := \frac{\partial}{\partial \theta(x)}\ell(\theta(x), g; w)$ be the partial derivative of $\ell$ with respect to it's first argument, $\theta(x)$.

In our work, we consider a general definition of calibration error that holds for any loss involving a nuisance component. This general notion of calibration, which is similar to notions that have been considered in \citet{noarov2023scope, gopalan2024swap, foster1999regret}, and \citet{globus2023multicalibration}, implies an estimator cannot be ``improved'' on any level set of its prediction.

\begin{definition}
\label{def:cal}
Let $\theta : \calX \rightarrow \R$ be an estimator, $\ell : \R \times \calG \times \calZ \rightarrow \R$ a nuisance-dependent loss function, and $g \in \calG$ a fixed nuisance parameter. We define the $L^2$ calibration error of $\theta$ with respect to $\ell$ and $g$ to be
\begin{align*}
\Cal(\theta, g) &:= 
\E\left(\E\left[\partial \ell(\theta, g; Z) \mid \theta(X)\right]^2\right)^{1/2} \\
&= \left(\int_{\calX} \E[\partial \ell(\theta, g; Z) \mid \theta(X) = \theta(x)]^2 P_X(dx)\right)^{1/2}.
\end{align*}
We say $\theta$ is perfectly calibrated if $\Cal(\theta, g_0) = 0$, where $g_0$ is the true, unknown nuisance parameter.
\end{definition}

In the special case where the loss $\ell(\nu, y)$ is the squared loss and doesn't involve a nuisance component, we recover a more classical notion of calibration error.

\begin{definition}[Classical calibration error]
\label{def:classic}
Let $\theta : \calX \rightarrow \R$ be a fixed estimator, and let $(X, Y) \sim P$, where $P$ is some arbitrary distribution on $\calX \times \R$. The \textit{classical $L^2$ calibration error} is defined by
\[
\Cal(\theta; P) := \int_{\calX}\left(\theta(x) - \E_{P_{Y}}[Y \mid \theta(X) = \theta(x)]\right)^2 P_X(dx),
\]
where we make dependence on the underlying distribution $P$ clear for convenience. 

\end{definition}

We will leverage the classical $L^2$ calibration error in the sequel when reasoning about the convergence of our sample splitting algorithm. The key for now is that Definition~\ref{def:cal} generalizes the above definition to arbitrary losses involving a nuisance component.

We are always interested in controlling $\Cal(\theta, g_0)$. In words, $\theta$ is calibrated if, on the level set $\{x \in \calX : \theta(x) = \nu\}$, there is no constant value $\omega \in \R$ we can switch the prediction $\theta(x)$ to to obtain lower loss. We can gleam further semantic meaning from looking at several additional examples below.

\begin{example}
\label{eg:cal}
Below, we almost always assume observations are of the form $Z = (X, A, Y)$, where $X$ are covariates, $A \in \{0, 1\}$ or $[0, 1]$ indicates treatment, and $Y = \sum_{a}Y(a)\mathbbm{1}_{A = a}$ indicates an outcome. We assume conditional ignorability of the treatment given covariates. The one exception is for conditional local average treatment effects, when assigned treatment may be ignored by the subject. 
\begin{enumerate}
    \item{\textbf{Conditional Average Treatment Effect:}} Recalling from earlier that the CATE $\theta_{\ATE}(x) := \E[Y(1) - Y(0) \mid X = x]$ can be specified as the conditional minimizer of the doubly-robust loss $\ell_{\ATE}$ given by $\ell_{\ATE}(\nu, (\mu, \pi); z) := \frac{1}{2}\left(\nu - \chi_{\ATE}((\mu, \pi); z)\right)^2$ with $\chi_\ATE$ given by
    \begin{equation}
    \label{eq:pseudo_CATE}
    \chi_{\ATE}((\mu, \pi); z) := \mu(1, x) - \mu(0, x) + \left(\frac{a}{\pi(x)} - \frac{1 - a}{1 - \pi(x)}\right)(y - \mu(a, x)).
    \end{equation}
    Here, the true nuisances are $\pi_0(x) := \P(A = 1 \mid X = x)$ and $\mu_0(a, x) := \E[Y \mid X=x, A = a]$. It is clear perfect calibration becomes $\E[Y(1) - Y(0) \mid \theta(X)] = \theta(X)$.  
    \item{\textbf{Conditional Average Causal Derivative:}} In a setting where treatments are not binary, but rather continuous (e.g.\ $A \in [0, 1]$), it no longer makes sense to consider treatment effects as a difference. Instead, we can consider the \textit{conditional average causal derivative}, which is defined by
    \[
    \theta_{\ACD}(x) := \E[\partial_a Y(A) \mid X]
    \]
    $\theta_{\ACD}$ is in fact the conditional minimizer of the loss $\ell_{\ACD}$ given by
    \[
    \ell_{\ACD}(\theta, g; z) = \frac{1}{2}(\theta(x) - \chi_{\ACD}(g; Z))^2 \qquad \chi_{\ACD}(g; Z) = \partial_a \mu(a, x) + \frac{\partial_a \pi(a \mid x)}{\pi(a \mid x)}(y - \mu(a , x))
    \]
    where $g_0 = (\mu_0, \pi_0)$, $\mu_0(a, x) = \E[Y \mid X = x, A = a]$ is again the expected outcome mapping, $\pi_0(a \mid x)$ is the \textit{density} of the treatment given covariates (we could alternatively directly estimate the nuisance $\zeta_0(a, x) = \frac{\partial_a \pi_0(a \mid x)}{\pi_0(a \mid x)}$ --- more on this later). Naturally, $\theta$ is perfectly calibrated with respect to $\ell_{\ACD}$ if 
    \[
    \theta(X) = \E\left(\partial_a \E[Y(A) \mid X] \mid \theta(X)\right).
    \]

    \item{\textbf{Conditional Local Average Treatment Effect:}} In settings with non-compliance, the prescribed treatment $D \in \{0, 1\}$ given to an individual may not be equivalent to the received treatment $A \in \{0, 1\}$. Formally, we have $Z = (X, D, A, Y)$, where we assume $A = A(1) D + A(0)(1 - D)$, and $Y = A Y(1) + (1 - A) Y(0)$, where $A(d), Y(a)$ represent potential outcomes for treatment assignment $a, d \in \{0, 1\}$. We also assume monotonicty, i.e.\ that $D(1) \geq D(0)$, and that the propensity for the recommended treatment $\pi_0(x) := \P(D = 1 \mid X = x)$ is known. The parameter of interest here is
    \[
    \theta_{\LATE}(x) := \E[Y(1) - Y(0) \mid A(1) > A(0), X = x],
    \]
    which is identified (following standard computations, see \citet{lan2023causal}) as 
    \begin{equation}
    \label{eq:late_ident}
    \theta_{\LATE}(x) = \frac{\E[Y \mid A = 1, X = x] - \E[Y \mid A = 0, X = x]}{\E[D \mid A = 1, X = x ] - \E[D \mid A = 0, X = x]} =: \frac{p_0(x)}{q_0(x)}.
    \end{equation}
    It follows that $\theta_{\LATE}$ conditionally minimizes the following somewhat complicated loss:
    \begin{align}
    \ell_{\LATE}(\theta, g; z) = \frac{1}{2}\left(\theta(x) - \chi_{\LATE}(g; z)\right)^2\qquad \text{where} \nonumber\\ 
    \chi_{\LATE}(g; z) := \frac{p(x)}{q(x)} + \frac{a(d - \pi_0(x))}{q(x)}\left(\frac{p(x)}{q(x)} - \frac{y(d- \pi_0(x))}{a(d - \pi_0(x))}\right).\label{eq:pseudo_LATE}
    \end{align}
    Calibration with respect to $\ell_{\LATE}$  becomes
    \[
    \theta(X) = \E\left[Y(1) - Y(0) \mid D(1) > D(0), \theta(X)\right].
    \]

    \item{\textbf{Conditional Quantile Under Treatment:}} Lastly, we consider the \textit{conditional $Q$th quantile under treatment}, which, assume $Y(1)$ admits a conditional Lebesgue density, is specified as $\theta_{\QTE}(x) = F^{-1}_{Y(1)}(Q \mid X =x)$\footnote{$F_{Y(1)}(\cdot \mid X = x)$ here denotes the conditional CDF of $Y(1)$ given covariates $X$}. More generally, the $Q$th quantile under treatment is specified as
    \[
    \theta_{\QTE}(x) \in \arg\min_{\nu \in \R}\E\left[\ell_{\QTE}(\nu, p_0; Z) \mid X = x\right].
    \]
    In the above, $\ell_{\QTE} : \R \times \calG \times \calZ \rightarrow \R$ denotes the $Q$-\textit{pinball loss}, which is defined as
    \[
    \ell_{\QTE}(\theta, p; z) := a p(x) (y - \theta(x))\left(Q - \mathbbm{1}_{y \leq \theta(x)}\right),
    \]
    where the true, unknown nuisance is the inverse propensity score $p_0(x) := \frac{1}{\pi_0(x)}$, where $\pi_0(x) := \P(A = 1 \mid X = x)$. A direct computation yields that calibration under $\ell_{\QTE}$ becomes
    \[
    \P\left(Y(1) \leq \theta(X) \mid \theta(X)\right) = Q.
    \]
\end{enumerate}
\end{example}

The first three losses considered above are ``easy'' to calibrate with respect to, as  $\ell_{\ATE}, \ell_{\ACD},$ and $\ell_{\LATE}$ are \textit{universally orthogonal} losses, a concept discussed in Section~\ref{sec:universal}. The pinball loss, on the other hand, is the quintessential example of a ``hard'' loss to calibrate with respect to. Handling more complicated losses like this is the subject of Section~\ref{sec:conditional}.

\section{Calibration for Universally Orthogonal Losses}
\label{sec:universal}

In this section, we develop algorithms for calibrating estimates $\theta$ of causal effects $\theta_0$ that are specified as conditional minimizers of \textit{universally orthogonal} losses $\ell$. Universal orthogonality, first introduced in \citet{foster2023orthogonal}, can viewed as a robustness property of losses that are ``close'' to squared losses. 
Heuristically, a loss is universally orthogonal if it is insensitive to small errors in estimating the nuisance functions \textit{regardless} of the current estimate on the conditional loss minimizer, i.e.\ Equation~\eqref{eq:neyman_ref} holds when $\theta_0$ is replaced by any function $\theta : \calX \rightarrow \R$. We formalize this in the following definition.

\begin{definition}[Universal Orthogonality]
\label{def:universal}
Let $\ell : \R \times \calG \times \calZ \rightarrow \R$ be a loss involving nuisance, and let $g_0$ denote the true nuisance parameter associated with $\ell$. We say $\ell$ is universally orthogonal, if for any $\theta : \calX \rightarrow \R$ and $g \in \calG$, we have
\[
D_{g}\E\left[\partial \ell(\theta, g_0; Z) \mid X\right](g - g_0) = 0.
\]
\end{definition}

By direct computation, one can verify that $\ell_{\ATE}, \ell_{\ACD},$ and $\ell_{\LATE}$ all satisfy Definition~\ref{def:universal}, whereas $\ell_{\QTE}$ does not --- this aligns with the idea of universally orthogonal losses behaving like squared loss. The following example gives a general class of losses that are universally orthogonal. 

\begin{example}
\label{eg:universal:gen_score}
    
Assume that we have a vector of nuisances $g_0 = (\eta_0, \zeta_0)$ for some $\eta_0, \zeta_0 : \calW \rightarrow \R^k$. Further, assume there is some function $m(\eta; z)$ such that\footnote{Here, $\zeta_0$ can be thought of a the (conditional) Riesz representer of the linear functional $D_\eta\E[m(\eta; Z) \mid X]$.}
\[
D_\eta\E[m(\eta; Z) \mid X](\eta - \eta_0) = \E[\langle \zeta_0(W), (\eta - \eta_0)(W)\rangle \mid X] \qquad \text{for any } g - g_0 \in \calG.
\]
Then, any loss $\ell(\theta, g; z)$ that obeys the score equation
\begin{equation}
\label{eq:universal:gen_score}
\partial\ell(\theta, g; z) = \theta(x) - \underbrace{m(\eta; z) + \corr(g; z)}_{=:\chi(g; z)}
\end{equation}
with $\E[\corr(g; Z) \mid X] = \E[\langle \zeta(W), (\eta_0 - \eta)(W)\rangle \mid X]$ will satisfy Definition~\ref{def:universal}. We think of $\chi(g; z) := m(\eta; z) - \corr(g; z)$ as representing a generalized pseudo-outcome that, on average, looks like the desired heterogeneous causal effect. In fact, any loss of the above form is equivalent to a loss that looks like $\ell(\nu, g; z) = \frac{1}{2}\left(\nu - \chi(g; z)\right)^2$. 

All universally orthogonal losses we have seen up until this point satisfy can be written in terms of pseudo-outcomes in the canonical form above. We note the particular settings of $\eta_0, \zeta_0, m,$ and $\corr$ for each of these losses in Table~\ref{tab:pseudo}. Typically, we assume the statistician will estimate $\eta_0$ and $\zeta_0$ by plugging in estimates for each unknown constituent function, e.g.\ by plugging in estimates for $\pi_0$ and $\mu_0$ in the CATE example.

\begin{table}[h]
    \centering
    \begin{tabular}{c|c|c|c|c}
    Loss & $\eta_0$ & $\zeta_0$ & $m(\eta; z)$ & $\corr(g; z)$ \\ 
    \hline 
    $\ell_{\ATE}$ & $\mu_0(a, x)$& $\left(\frac{a}{\pi_0(x)} - \frac{1 - a}{1 - \pi_0(x)}\right)$& 
    $\eta(1, x) - \eta(0, x)$
    & $\zeta(a, x) \cdot (y - \eta(a, x))$\\
    $\ell_{\ACD}$ & $\mu_0(a, x)$ & $\frac{\partial_a \pi_0(a \mid x)}{\pi_0(a \mid x)}$  & $\partial_a \mu(a, x)$ & $\zeta(a, x) \cdot (y - \eta(a, x))$ \\
    $\ell_{\LATE}$ & $\frac{p_0(x)}{q_0(x)}$ & $\frac{a (d - \pi_0(x))}{q_0(x)}$ & $\eta(x)$ & $\zeta(a, d, x)\cdot \left(\frac{y(d - \pi_0(x))}{a(d - \pi_0(x))} - \eta(x)\right)$  \\ 
    \end{tabular}
    \caption{Canonical representations of universally orthogonal losses in terms of base nuisances $\eta_0$, Riesz representers $\zeta_0$, and pseudo-outcome components $m(\eta; z)$ and $\corr(g; z)$.}
    \label{tab:pseudo}
\end{table}


\end{example}

Universal orthogonality allows to bound $\Cal(\theta, g_0)$ above by two decoupled terms.  The first, denoted $\err(g, g_0; \theta)$, intuitively measures the distance between some fixed nuisance estimate $g \in \calG$ and the unknown, true nuisance parameters $g_0 \in \calG$. This error is second order in nature, generally depending on the squared norm of the nuisance estimation error. The second term is $\Cal(\theta, g)$, and represents the calibration error of $\theta$ in a reality where the learned nuisances $g$ were actually the true nuisances. 
As mentioned in the introduction, we view our bound as a ``change of nuisance'' result (akin to change of measure), allowing the learner to pay an upfront price for nuisance misestimation and then subsequently reason about the calibration error under potentially incorrect, learned nuisances. We prove the following in Appendix~\ref{app:decomp}.

\begin{restatable}{theorem}{thmuniversal}
\label{thm:universal}
Let $\ell : \R \times \calG \times \calZ \rightarrow \R$ be universally orthogonal, per Definition~\ref{def:universal}. Let $g_0 \in \calG$ denote the true nuisance parameters associated with $\ell$. Suppose $D^2_g\E[\partial \ell(\theta, f; Z) \mid X](g - g_0, g - g_0)$ exists for any $g, f \in \calG$. Then, for any $g \in \calG$ and $\theta : \calX \rightarrow \R$, we have
\[
\Cal(\theta, g_0) \leq \frac{1}{2}\err(g, g_0; \theta) + \Cal(\theta, g),
\]
where $\err(g, h; \theta) := \sup_{f \in [g, h]}\sqrt{\E\left(\left\{D_g^2\E\left[\partial \ell(\theta, f; Z) \mid X\right](h - g, h - g)\right\}^2\right)}$.\footnote{For $f, h \in \calG$, we let the interval $[f, h] := \{\lambda(w) f(w) + (1 - \lambda(w))h(w) : \lambda(w) \in [0, 1]\}$.}

\end{restatable}

While the expression defining $\err(g, g_0; \theta)$ looks unpalatable, when $g_0 = (\eta_0, \zeta_0)$ (as in Example~\ref{eg:universal:gen_score}), this term can quite generally be bounded above by the cross-error in nuisance estimation, i.e.\ the quantity  $\|\langle \eta - \eta_0, \zeta - \zeta_0\rangle\|_{L^2(P_W)}$. The formal statement of this fact is presented in Proposition~\ref{prop:cross_error} below. More broadly, if the conditional Hessian satisfies $|D_g^2 \E[\partial \ell(\theta, f; Z) \mid X](\Delta, \Delta)|  \leq C \|\Delta\|_{L^2(P_W)}$ for any $\Delta \in \calG - \calG$ and $f \in \calG$, then we simply have $\err(g, g_0; \theta) \leq C\|g - g_0\|^2_{L^2(P_W)}$. These sorts of second order bound appears in other works on causal estimation~\citep{foster2023orthogonal, van2023causal, lan2023causal, chernozhukov2018double}.

We view the above result below as a major generalization of the main result (Theorem 1) of \citet{van2023causal}, which shows a similar bound for measuring the calibration error of estimates of conditional average treatment effects when calibration is performed according to isotonic regression.  Our result, which more explicitly leverages the concept of Neyman orthogonality, can be used to recover that of \citet{van2023causal} as a special case, including the error rate in nuisance estimation. 

\begin{restatable}{prop}{propcrosserror}
\label{prop:cross_error}
Suppose the loss $\ell$ satisfies the score condition outlined in Equation~\eqref{eq:universal:gen_score} and suppose $m(\eta; z)$ is linear in $\eta$. Then, we have
\[
\err(g, g_0; \theta) \leq 2\|\langle \eta - \eta_0, \zeta - \zeta_0\rangle\|_{L^2(P_W)},
\]
where $g_0 = (\eta_0, \zeta_0)$ represent the true, unknown nuisance parameters, and $g = (\eta, \zeta)$ represent arbitrary, fixed nuisance estimates. If instead one has $|D_g^2\E[\partial \ell(\theta, f; Z)(\Delta, \Delta)| \leq C \|\Delta\|_{L^2(P_W)}^2$ for any $\Delta$ and $\theta$, the one has
\[
\err(g, g_0; \theta) \leq C\|g - g_0\|_{L^2(P_W)}^2.
\]
\end{restatable}

\subsection{Sample Splitting Algorithm}

Throughout the remainder of this section, we make the following assumption on the the loss $\ell$. Any loss of the form presented in Equation~\eqref{eq:universal:gen_score} naturally satisfies the following assumption, and thus our example is applicable to $\ell_{\ATE}, \ell_{\ACD},$ and $\ell_{\LATE}$ described earlier.

\begin{ass}[Linear Score]
\label{ass:gen_universal}
We assume there is some function $\chi: \calG \times \calZ \rightarrow \R$ such that the loss $\ell : \R \times \calG \times \calZ \rightarrow \R$ satisfies
\[
\E[\partial \ell(\theta, g; Z) \mid X] = \theta(X) - \E[\chi(g; Z) \mid X]
\]
for any $g \in \calG$ and $\theta : \calX \rightarrow \R$. 
\end{ass}

The largely theoretic bound presented in Theorem~\ref{thm:universal} suggests a natural algorithm for performing causal calibration. First, a learner should use some fraction (say half) of the data to produce nuisance estimate $\wh{g}$ using any black-box algorithm. Then, the learner should transform second half of the data into generalized pseudo-outcomes $\chi(\wh{g}; Z)$ using the learned nuisances. Finally, they should apply some off-the-shelf calibration algorithm to the transformed data points. We formalize this in Algorithm~\ref{alg:sample-split-universal}.

\begin{algorithm}[ht]
   \caption{Sample-Splitting for Universally Orthogonal Losses}
   \label{alg:sample-split-universal}
\begin{algorithmic}[1]
   \Require Samples $Z_1, \dots, Z_{2n} \sim P_Z$, nuisance alg.\ $\calA_1$, calibration alg.\ $\calA_2$, estimator $\theta$.
   \State Estimate nuisances: $\wh{g} \gets \calA_1(Z_{1:n})$.
   \State Compute pseudo-outcomes: $\chi_{n + 1}, \dots, \chi_{2n} := \chi(\wh{g}; Z_{n + 1}), \dots, \chi(\wh{g}; Z_{2n})$.
   \State Run calibration: $\wh{\theta} \gets \calA_2(\theta, (X_m, \chi_m)_{m = n + 1}^{2n})$
   \Ensure Calibrated estimator $\wh{\theta}$.
\end{algorithmic}
\end{algorithm}

Algorithm~\ref{alg:sample-split-universal} generalizes the main algorithm in \citet{van2023causal} (Algorithm 1) to allow for the calibration of an estimate of any heterogeneous causal effect using any off-the-shelf nuisance estimation and calibration algorithms. To make efficient use of the data, we recommend instead running the cross calibration procedure outlined in Algorithm~\ref{alg:cross-cal-universal}. We do not provide a convergence analysis for this algorithm.

\begin{algorithm}[ht]
   \caption{Cross Calibration for Universally Orthogonal Losses}
   \label{alg:cross-cal-universal}
\begin{algorithmic}[1]
   \Require Samples $\calD = Z_1, \dots, Z_{n} \sim P_Z$, nuisance alg.\ $\calA_1$, calibration alg.\ $\calA_2$, estimator $\theta$.
   \State Split $\calI := [n]$ into $K$ equally-sized folds $\calI_1, \dots, \calI_K$.
   \State Define $\calD_k := \{Z_i : i \in \calI_k\}$ for all $k \in [K]$.
   \For{$k \in 1, \dots, K$}
   \State Estimate nuisances: $\wh{g}^{(-k)} \gets \calA_1(\calD\setminus \calD_k)$.
   \State Compute pseudo-outcomes $\chi_i := \chi(\wh{g}^{(-k)}; Z_i)$ for $i \in \calI_k$. 
   \EndFor
   \State Run calibration: $\wh{\theta} \gets \calA_2(\theta, (X_m, \chi_m)_{m = 1}^{n})$
   \Ensure Calibrated estimator $\wh{\theta}$.
\end{algorithmic}
\end{algorithm}

We now prove convergence guarantees for Algorithm~\ref{alg:sample-split-universal}. We start by enumerating several assumptions needed to guarantee convergence.

\begin{ass}
\label{ass:alg_univ}
Let $\calA_1 : \calZ^\ast \rightarrow \calG$ be a nuisance estimation algorithm taking in an arbitrary number of points, and let $\calA_2 : \Theta \times (\calX \times \R)^\ast \rightarrow \Theta$ be a calibration algorithm taking some initial estimator and an arbitrary number of covariate/label pairs. We assume
\begin{enumerate}
    \item For any distribution $P_Z$ on $\calZ$, $Z_1, \dots, Z_n \sim P$ i.i.d., and failure probability $\delta_1 \in (0, 1)$, we have
    \[
    \err(\wh{g}, g_0; \theta) \leq \rate_1(n, \delta_1; P_Z),
    \]
    where $\wh{g} \gets \calA_1(Z_{1:n})$ and $\rate_1$ is some rate function.
    \item For any distribution $Q$ on $\calX \times \R$, $(X_1, Y_1), \dots, (X_n, Y_n) \sim Q$ i.i.d., initial estimator $\theta : \calX \rightarrow \R$, and failure probability $\delta_2 \in (0, 1)$, we have
    \[
    \Cal(\wh{\theta}; Q) \leq \rate_2(n, \delta_2; Q),
    \]
    where $\wh{\theta} \gets \calA_2(\theta, \{(X_i, Y_i)\}_{i=1}^n)$ and $\rate_2$ is some rate function.
\end{enumerate}
\end{ass}

We briefly parse the above assumptions. For the first assumption, whenever $g_0 = (\eta_0, \zeta_0)$ and $\err(\wh{g}, g_0; \theta) \lesssim \|(\eta - \eta_0)(\zeta - \zeta_0)\|_{L^2(P_W)}$ or $\err(\wh{g}, g_0; \theta) \lesssim \|\wh{g} - g_0\|_{L^2(P_W)}^2$, one can directly apply ML, non-parametric, or semi-parametric methods to estimate the unknown nuisances. For instance, if the nuisances are assumed to assumed to satisfy Holder continuity assumptions or are assumed to belong to a ball in a reproducing kernel Hilbert space, one can apply classical kernel smoothing methods or kernel ridge regression respectively to estimate the unknown parameters~\citep{van2023causal, tsy2009est, wainwright2019high} to obtain optimal rates. 
Further, many well-known calibration algorithms satisfy the second assumption, often in a manner that doesn't depend on the underlying distribution $Q$. For instance, results in \citet{gupta2021distribution} on $L^\infty$ calibration error bounds directly imply that if $\calA_2$ is taken to be uniform mass/histogram binning, then the rate function $\rate_2$ can be taken as $\rate_2(n, \delta; Q) = O\left(\|Y\|_\infty\sqrt{\frac{B\log(B/\delta)}{n}}\right)$, where $B$ denotes the number of bins/buckets, $\|Y\|_\infty$ denotes the essential supremum of $Y$ (which will be finite so long as nuisances and observations are bounded) and the unknown constant is independent of $n$, $Q$, and $\delta$. Likewise, the convergence in probability results for isotonic calibration proven by \citet{van2023causal} can naturally be extended to high-probability guarantees using standard concentration of measure results.

\begin{restatable}{theorem}{thmalguniversal}
\label{thm:alg_universal}
Suppose $\ell : \R \times \calG \times \calZ \rightarrow \R$ is an arbitrary, universally orthogonal loss satisfying Assumption~\ref{ass:gen_universal}. Let $\calA_1, \calA_2,$ and $\chi$ satisfy Assumption~\ref{ass:alg_univ}. Then, with probability at least $1 - \delta_1 - \delta_2$, the output $\wh{\theta}$ of Algorithm~\ref{alg:sample-split-universal} run on a calibration dataset of $2n$ i.i.d.\ data points $Z_{1}, \dots, Z_{2n} \sim P$ satisfies
\[
\Cal(\wh{\theta}, g_0) \leq \frac{1}{2}\rate_1(n, \delta_1; P_Z) + \rate_2(n, \delta_2; P^\chi_{\wh{g}}),
\]
where $P^\chi_{\wh{g}}$ denotes the (random) distribution of $(X, \chi(\wh{g}; Z))$ where $Z \sim P_Z$ and again $X \subset Z$.
\end{restatable}

We prove the above theorem in Appendix~\ref{app:alg}. The above result can be thought of as an analogue of Theorem 1 of \citet{foster2023orthogonal}, which shows a similar bound on excess parameter risk, and also a generalization of Theorem 1 of \citet{van2023causal}, which shows a similar bound when isotonic regression is used to calibrate CATE estimates. 

We note that while we state our bounds and assumptions in terms of high-probability guarantees, we could have equivalently assumed, $\err(\wh{g}, g_0; \theta) = O_\P(\rate_1(n; P_Z))$ and $\Cal(\wh{\theta}; Q) = O_\P(\rate_2(n))$ for appropriately chosen rate functions $\rate_1$ and $\rate_2$\footnote{Here, $\rate_2(n)$ would have to be a function satisfying independent of the distribution of pseudo-outcomes to obtain convergence in probability guarantees. Such functions exists for isotonic calibration and histogram binning}. This, for instance, would be useful if one wanted to apply the results on the convergence of isotonic regression due to \citet{van2023causal}, who show $\Cal(\theta; P) = O_\P\left(n^{-1/3}\right)$. 

Lastly, we note that it is desirable for calibration algorithms to possess a ``do no harm'' guarantee, which ensures that the risk of the calibrated parameter $\wh{\theta}$ is not much larger than the risk of original parameter. We present such a guarantee in Theorem~\ref{thm:risk_bd} in Appendix~\ref{app:loss_red}, which follows using standard risk bounding techniques due to \citet{foster2023orthogonal}.

\section{Calibration for Conditionally Orthogonal Losses}
\label{sec:conditional}

We now consider the more challenging problem where the causal effect $\theta_0(X)$ is not the minimizer of a universally orthogonal loss. To aid in our exposition, we introduce \textit{calibration functions}. In short, the calibration function gives a canonical choice of a post-processing $\wh{\theta} = \tau \circ \theta: \calX \rightarrow \R$ such that $\Cal(\wh{\theta}; g_0) = 0$. While computing the calibration function exactly requires knowledge of the data-generating distribution, it can be approximated in finitely-many samples.

\begin{definition}[\textbf{Calibration Function}]
\label{def:cal_func}
Given any $\theta : \calX \rightarrow \R$ and $g \in \calG$, we define the \textit{calibration function for $\theta$ at $g$} as the mapping $\gamma_\theta(\cdot; g) : \calX \rightarrow \R$ given by
\[
\gamma_\theta(x; g) := \arg\min_\nu\E[\ell(\nu, g; Z) \mid \theta(X) = \theta(x)].
\]
In particular, when $g = g_0$, we call $\gamma_\theta^\ast := \gamma_\theta(\cdot; g_0)$ the \textit{true calibration function}.

\end{definition}

As hinted, first-order optimality conditions alongside the tower rule for conditional expectations imply that $\E[\partial \ell(\gamma_\theta(\cdot; g), g; Z) \mid \gamma_\theta(X; g)] = 0$ for any $g \in \calG$. This, in particular, implies that $\gamma_\theta^\ast$ is perfectly calibrated. As an example, when a loss satisfies Assumption~\ref{ass:gen_universal}, $\gamma_\theta^\ast(x, g) = \E[\chi(g; Z) \mid \theta(X) = \theta(x)]$.

 We now ask under what general assumptions on $\ell$ can we achieve calibration. In analogy with \citet{foster2023orthogonal}, who consider the general task of empirical risk minimization in the presence of nuisance, the hope would be to weaken the assumption of universal orthogonality to that of \textit{Neyman orthogonality}. Two commonly-used definitions for Neyman orthogonality (a marginal version and a version conditional on covariates) are provided below. 

\begin{definition}[Neyman Orthogonality]
\label{def:neyman}
 We say $\ell$ is \textit{Neyman orthogonal conditional on covariates $X$} (or marginally) if, for $g \in \calG$, we have 
 \[
 D_g \E[\partial \ell(\theta_0, g_0; Z) \mid X](g - g_0) = 0 \qquad (\text{respectively, } D_g\E[\partial \ell(\omega_0, h_0; Z)](g - h_0) = 0),
 \]
where $\theta_0(x) := \arg\min_{\nu}\E[\ell(\nu, g_0; Z) \mid X = x]$ and $g_0$ denotes the true nuisances (respectively $\omega_0 := \arg\min_\nu\E[\ell(\nu, h_0; Z)]$ and $h_0$ for the latter). 
\end{definition}

Neyman orthogonality is useful for a task such as risk minimization because it allows the statistician to relate the risk $\E[\ell(\wh{\theta}, g_0; Z)]$ under the true nuisances to the risk under the computed nuisances $\E[\ell(\wh{\theta}, \wh{g}; Z)]$ up to second order errors. Why do we need two separate conditions on the loss $\ell$? In general, the conditions in Definition~\ref{def:neyman} are not equivalent. To illustrate this, we can look at the example of the conditional/marginal quantile under treatment (We will let $\theta_{\QTE}(X)$ denote the former and $\omega_{\QTE}$ the latter). Recalling that we defined the pinball loss $\wt{\ell}_{\QTE}(\nu, p; z) := a p(x) (y - \nu)(Q - \mathbbm{1}_{y \leq \nu})$, it is not hard to see that we have
\[
\theta_{\QTE}(x) = \arg\min_\nu\E[\wt{\ell}_{\QTE}(\nu, p_0; Z) \mid X = x] \quad \text{and} \quad \omega_{\QTE} = \arg\min_\nu \E[\wt{\ell}_{\QTE}(\nu, p_0; Z)],
\]
where $p_0(x) := \P(A = 1 \mid X = x)^{-1}$ denotes the inverse propensity. Straightforward computation yields that $\wt{\ell}_{\QTE}$ is Neyman orthogonal conditional on covariates, but \text{not} marginally orthogonal. However, as noted in \citet{kallus2019localized}, one can define the more complicated loss $\ell_{\QTE}(\nu, (p, f); z)$ by performing a first order correction:
\begin{equation}
\label{eq:qte_loss}
\ell_{\QTE}(\nu, (p, f); z) := \wt{\ell}_{\QTE}(\nu, p; z) - \theta(x)\cdot\underbrace{(ap(x)(f(x) - Q) -  f(x) + Q)}_{=:\corr((p, f); z)}
\end{equation}
where $f_0(x) := \P(Y(1) \leq \omega_{\QTE} \mid X = x)$ is an additional nuisance that must be estimated from the data. One can check that $\ell_{\QTE}$ satisfies the second condition of Definition~\ref{def:neyman}, i.e.\ marginal Neyman orthogonality.

In calibration, we care about the quality of a model \textit{conditional on its own predictions}. More specifically, given any initial model $\theta(X)$, the goal of any calibration algorithm (e.g.\ histogram binning, isotonic regression) is to compute the calibration function $\gamma_{\theta}^\ast$ (Definition~\ref{def:cal_func}).\footnote{We remind the reader that $\gamma_\theta^\ast(x) := \arg\min_\nu\E[\ell(\nu, g_0; w) \mid \theta(X) = \theta(x)]$} If $\theta(X)$ were a constant function, then we would have $\gamma_\theta^\ast(x) \equiv \omega_0$, and thus we would want to leverage a loss satisfying marginal orthogonality. Likewise, if $\theta(X)$ were roughly of the same complexity as $\theta_0(X)$, we may want to leverage the a loss $\ell$ satisfying the form of Neyman orthogonality conditional on covariates $X$. In general, the complexity of the initial estimate $\theta(X)$ will interpolate between these two extremes.

For a variant of Neyman orthogonality to thus be useful, we would need to cross-derivative to vanish (a) when evaluated at the calibration function $\gamma_{\theta}^\ast$ instead of $\theta_0(X)$ or $\omega_0$ and (b) when the expectation is taken conditionally on the prediction $\theta(X)$ instead of either conditionally on $X$ or marginally.
The extra structure provided by universal orthogonality allowed us to side-step this issue as the cross-derivative of the loss vanished when evaluated at \textit{any} estimate of $\theta_0$ so long as nuisances were estimated correctly.
The following, quite technical condition will give us the structure we need to perform calibration of estimates of more general heterogeneous causal effects.

\begin{definition}[Conditional Orthogonality]
\label{def:conditional}
Suppose $\wt{\ell}(\theta, \eta; z)$ is some initial loss with true nuisance $\eta_0$. Define the ``corrected'' loss $\ell$ by
\[
\ell(\theta, (\eta, \zeta); z) := \wt{\ell}(\theta, \eta; z) - \theta(x)\cdot\corr((\eta, \zeta); z),
\]
where $\corr((\eta, \zeta); z)$ is any correction term satisfying $\E[\corr((\eta_0, \zeta); z)\mid X] = 0$. Then, we say $\ell$ is conditionally orthogonal if, for any $\varphi : \calX \rightarrow \R$, there exists $\zeta_\varphi \in L^2(P_X)$ such that
\[
D_g\E[\partial \ell(\gamma_\varphi^\ast,  g_\varphi; Z) \mid \varphi(X)](g - g_0) = 0,
\]
for all $g \in \calG$, where $g_\varphi := (\eta_0, \zeta_\varphi)$.

\end{definition}



Definition~\ref{def:conditional} may be difficult to parse, but we can work through several examples to gain some intuition.
Returning to the example of conditional quantile under treatment and the (corrected) pinball loss $\ell_{\QTE}(\nu, (p, f); z)$ defined above, we simply took the correction term to be $\corr((p, f); z) := a p(x)(f(x) - Q) - f(x) + Q$. From a straightforward calculation, one can check $\ell_{\QTE}$ satisfies Definition~\ref{def:conditional} with additional nuisance $f_\varphi$ given by $f_\varphi(x) := \P(Y(1) \leq \gamma_{\varphi}^\ast(X) \mid X = x)$.

More broadly, given some initial loss $\wt{\ell}$, one can use the Riesz representation theorem 
to obtain $\zeta_\varphi : \calW \rightarrow \R$ satisfying
\[
D_\eta \E[\partial \wt{\ell}(\gamma_{\varphi}^\ast, \eta_0; Z) \mid \varphi(X)](\eta - \eta_0) = \E[\langle \zeta_\varphi(W), (\eta - \eta_0)(W) \rangle \mid \varphi(X)]
\]
almost surely. Then, if we can find some variable $U \subset Z$ such that $\E[U \mid W] = \eta_0(W)$, we can simply take $\corr((\eta, \zeta); z) := \langle \zeta(w), \eta(w) - u\rangle$, which gives us a ``corrected'' loss
\[
\ell(\theta, (\eta, \zeta); z) := \wt{\ell}(\theta, \eta; z) - \theta(x)\cdot\langle \zeta(w), \eta(w) - u\rangle.
\]


We conclude by pointing out the following observations about calibration with respect to such losses $\ell$, which follows immediately from Definition~\ref{def:conditional}.

\begin{corollary}
\label{cor:err_inv}
Suppose a loss $\ell$ satisfies Definition~\ref{def:conditional}. Then, the following hold:
\begin{enumerate}
    \item The calibration function \[
     \gamma_\varphi(x; (\eta, \zeta)) := \arg\min_\nu\E[\ell(\nu, (\eta, \zeta); Z) \mid \varphi(X) = \varphi(x)]
     \]
    doesn't depend on $\zeta$ when $\eta = \eta_0$. Thus, we can write $\gamma_\varphi^\ast(x) = \gamma_\varphi(x; (\eta_0, \zeta))$ for any $\zeta$.
    \item The calibration error of any estimate $\theta : \calX \rightarrow \R$, given by
    \[
    \Cal(\theta, (\eta, \zeta)) := \left(\int_\calX \E[\partial \ell(\theta, (\eta, \zeta); Z) \mid \theta(X)]^2 P_X(dx)\right)^{1/2}
    \]
    also does not depend on $\zeta$ when $\eta = \eta_0$. Thus, we can write $\Cal(\theta, \eta_0)$ in place of $\Cal(\theta, (\eta_0, \zeta))$ for any $\zeta$. 
    \item Lastly, not only do $\ell$ and $\wt{\ell}$ posses the same conditional minimizer when evaluated at $\eta_0$ (regardless of the choice of $\zeta$), but we also have 
    \[
    \Cal(\theta, (\eta_0, \zeta)) = \wt{\Cal}(\theta, \eta_0),
    \]
    where $\theta : \calX \rightarrow \R$ is arbitrary and $\wt{\Cal}$ denotes the calibration error under $\wt{\ell}$.
\end{enumerate}
\end{corollary}

\subsection{A General Bound on Calibration Error}
We now prove a decoupled bound on the calibration error $\Cal(\theta, g_0)$ under the assumption that $\ell$ is conditionally orthogonal. This bound serves as a direct analogue of the one presented in Theorem~\ref{thm:universal}, just in a more general setting. As we will see, bounding the calibration error of losses that are not universally orthogonalizable is a much more delicate task.

To prove our result, we will need place convexity assumptions on the underlying loss $\ell$. We note that these convexity results are akin to those made in existing works, namely in the work of \citet{foster2023orthogonal}.

\begin{ass}
\label{ass:strong_conv}
We assume that the loss function conditioned on covariates $X$ is $\alpha$-strongly convex, i.e.\ for any $v\in \R$ and any $g \in \calG$, we have
\[
\E\left[\partial^2\ell(v, g; Z) \mid X\right] \geq \alpha.
\]

\end{ass}

\begin{ass}
\label{ass:smooth}
We assume that the loss function conditioned on covariates $X$ is $\beta$-smooth, i.e.\ for any $v\in \R$ and any $g \in \calG$, we have
\[
\E\left[\partial^2\ell(v, g; Z) \mid X\right] \leq \beta.
\]

\end{ass}

We now state the main theorem of this section. The bound below appears largely identical to the one presented in Theorem~\ref{thm:universal} modulo two minor differences. First, we pay a multiplicative factor of $\beta/\alpha$ in both of the decoupled terms, which ultimately is just the condition number of the loss $\ell$. Second, the error term  $\err(g, g_\theta; \gamma_\theta^\ast)$ is evaluated at the calibration function $\gamma_\theta^\ast(x) := \arg\min_{\nu}\E[\ell(\nu, g_0; Z) \mid \theta(X) = \theta(x)]$  instead of the parameter estimate $\theta$. This difference is due to the fact that, in the proof of Theorem~\ref{thm:conditional}, we must perform a functional Taylor expansion around $\gamma_\theta^\ast$ in order to invoke the orthogonality condition. This subtlety was absent in the case of universal orthogonality, as $\ell$ was insensitive to nuisance misestimation \textit{for any} parameter estimate $\theta$. We ultimately view this difference as minor, as for many examples (e.g.\ the pinball loss $\ell_{\QTE}$) the dependence on $\gamma_\theta^\ast$ vanishes. We prove Theorem~\ref{thm:conditional} in Appendix~\ref{app:decomp:conditional}. 
\begin{restatable}{theorem}{thmconditional}
\label{thm:conditional}
Let $\ell$ be a conditionally orthogonal loss (Definition~\ref{def:conditional}) that is $\alpha$-strong convex (Assumption~\ref{ass:strong_conv}) and $\beta$-smooth (Assumption~\ref{ass:smooth}). Suppose $D_g^2\E[\partial\ell(\gamma_\theta^\ast, f; Z) \mid X](g - g_0, g - g_0)$ exists for $f, g \in \calG$. Then, for any estimate $\theta : \calX \rightarrow \R$ and nuisance parameter $g = (\eta, \zeta)$, we have
\[
\Cal(\theta, \eta_0) \leq \frac{\beta}{2\alpha}\err(g, g_\theta; \gamma_\theta^\ast) + \frac{\beta}{\alpha}\Cal(\theta, g),
\]
where $g_\theta = (\eta_0, \zeta_\theta)$ are the true, unknown nuisance functions, $\gamma_\theta^\ast$ is the calibration function associated with $\theta$, and $\err(g, g_\theta; \gamma_\theta^\ast)$ is as defined in Theorem~\ref{thm:universal}.
\end{restatable}

Although the bound in Theorem~\ref{thm:conditional} looks similar in spirit to the one presented in Theorem~\ref{thm:universal}, there still remain questions to answer. For instance, what does the condition number $\beta/\alpha$ look like for practically-relevant losses? Likewise, will the error term $\err(g, g_\theta; \gamma_\theta^\ast)$ simplify into a cross-error term as in the case of universally orthogonalizable losses? We interpret Theorem~\ref{thm:conditional} by spending some time looking at the example of the pinball loss $\ell_{\QTE}$.

\begin{example}
First, for any fixed quantile $Q$, we assess the strong convexity/smoothness properties of $\ell_{\QTE}$. Let $p : \calX \rightarrow \R_{\geq 0}$ represent any inverse-propensity estimate, and let $\pi_0$ represent the true propensity score. Assume $Y(1)$ admits a conditional density $f_{Y(1)}(y \mid x)$ with respect to the Lebesgue measure on $\R$. Straightforward calculation yields 
\[
\partial^2\E[\ell_{\QTE}(\theta, p; Z) \mid X] = f_{Y(1)}(\theta(x) \mid x).
\]
Thus, if $l \leq f_{Y(1)}(y \mid x) \leq u$ for all $y \in \R, x \in \calX$ and $\epsilon < p_0(x)^{-1} , p(x)^{-1} \leq 1 - \epsilon$ for all $x \in \calX$ for some $0 < \epsilon < 1/2$, then we have:
\[
\frac{\epsilon^2}{(1 - \epsilon)^2}l \leq \partial^2\E[\ell_{\QTE}(\theta, p; Z) \mid X] \leq \frac{(1 - \epsilon)^2}{\epsilon^2}u,
\]
i.e.\ that $\ell_{\QTE}$ satisfies Assumption~\ref{ass:smooth} with $\beta = \frac{1 - \epsilon}{\epsilon}u$ and Assumption~\ref{ass:strong_conv} with $\alpha = \frac{\epsilon}{1 - \epsilon}l$. 

We can further interpret the error term $\err(g, g_\theta; \gamma_\theta^\ast)$ in the case of the pinball loss, where again $g = (\eta, f)$. In particular, straightforward calculation yields
\[
\E[\partial \ell_{\QTE}(\theta, g; Z) \mid X] = \frac{p(X)}{p_0(X)}(\P(Y(1) \leq \theta(X) \mid X) - Q) - \E\left[A(f(X) - Q)(p(W) - p_0(W))\mid X\right].
\]
As the first term is linear in the nuisance estimate $p$ and doesn't depend on $f$, its second Gateaux derivative (with respect to $g$) is identically zero. Thus, the error term does not depend on $\gamma_\theta^\ast$, and thus we can write $\err(g, g_\theta)$ instead. Further, using the same analysis as in Proposition~\ref{prop:cross_error} and writing $\zeta(a, x) = a(f(x) - Q)$. we have that
\begin{align*}
\frac{1}{2}\err(g, g_\theta) &\leq \|(p - p_0)(\zeta - \zeta_\theta)\|_{L^2(P_W)} \\
&= \left\|A(f(X) - f_\theta(X))(p(X) - p_0(X))\right\|_{L^2(P_W)},
\end{align*}
where we recall that $f_\theta(x) := \P(Y(1) \leq \gamma_\theta^\ast \mid X  = x)$. Thus, even in the general case of conditional orthogonality, we can often obtain simple looking bounds on the error in nuisance estimation.
\end{example}

\subsection{A Sample Splitting Algorithm}

We conclude the section by presenting two algorithms for performing causal calibration with respect to conditionally orthogonal losses. As in Section~\ref{sec:universal}, we first present a sample splitting algorithm (Algorithm~\ref{alg:sample-split-cond}) that enjoys finite sample convergence guarantees. We then present a corresponding cross-calibration algorithm that is likely more useful in practice. 
\begin{algorithm}[ht]
   \caption{Sample-Splitting for Conditionally Orthogonal Losses}
   \label{alg:sample-split-cond}
\begin{algorithmic}[1]
   \Require Samples $Z_1, \dots, Z_{2n} \sim P_Z$, nuisance alg.\ $\calA_1$, general loss calibration alg.\ $\calA_2$, estimator $\theta$, loss involving nuisance $\ell$
   \State Estimate nuisances: $\wh{g} \gets \calA_1(Z_{1:n})$.
   \State Compute loss partial-evaluations: $\ell_{n + 1}, \dots, \ell_{2n} := \ell(\cdot, \wh{g}; Z_{n + 1}), \dots, \ell(\cdot, \wh{g}; Z_{2n})$
   \State Run calibration: $\wh{\theta} \gets \calA_2(\theta, (X_m, \ell_m)_{m = n + 1}^{2n})$
   \Ensure Calibrated estimator $\wh{\theta}$.
\end{algorithmic}
\end{algorithm}

Algorithm~\ref{alg:sample-split-cond} is essentially a generalization of Algorithm~\ref{alg:sample-split-universal} to general losses. The key difference is that we can no longer compute pseudo-outcomes for general losses. Instead, we assume that the calibration algorithm $\calA_2$ passed to Algorithm~\ref{alg:sample-split-universal} can calibrate ``with respect to general losses $\ell$''. What does this mean? Many calibration algorithms such as linear calibration, Platt scaling, and isotonic calibration, compute a mapping $\wh{\tau}$ satisfying
\begin{equation}
\label{eq:cal_loss}
\wh{\tau} \in \arg\min_{\tau \in \calF}\sum_{m = 1}^n\ell\left((\tau\circ\theta)(X_i), Y_i\right)
\end{equation}
where $(X_1, Y_1), \dots, (X_n, Y_n)$ denotes a calibration dataset,  $\calF \subset \{f : \R \rightarrow \R\}$ is some appropriately-defined class of functions, and $\ell$ is an appropriately chosen class of functions. Table~\ref{tab:cal_loss} below outlines the choices of $\calF$ and $\ell$ for common calibration algorithms. 

\begin{table}[h]
    \centering
    \begin{tabular}{c|c|c}
    Algorithm & Loss $\ell$ & Function class $\calF$ \\
    \hline 
    Isotonic Calibration & $\ell(\nu, y) = \frac{1}{2}(\nu - y)^2 $& $\calF = \{\tau(x) \text{ is non-decreasing}\}$\\
    Linear Calibration & $\ell(\nu, y) = \frac{1}{2}(\nu - y)^2$ & $\calF = \{ \tau(x) = a x + b,\;a, b \in \R\}$ \\
    Histogram binning & $\ell(\nu, y) = \frac{1}{2}(\nu - y)^2$ & {\footnotesize$\calF = \{\tau(x) = \sum c_b \mathbbm{1}_{\theta(x) \in [\theta(X)_{(b- 1)n/B)}, \theta(X)_{(bn/B)}}\}$} \\ 
    Platt Scaling & {\footnotesize $\ell(\nu, y) = -y\log(\nu) - (1 - y)\log(1 - \nu)$} & $\calF = \{ \tau(x)  = \frac{1}{1 + \exp(a x + b)}, a, b\in\R\}$
    \end{tabular}
    \caption{Representations of classical calibration algorithms in the form outlined in Equation~\eqref{eq:cal_loss}. }
    \label{tab:cal_loss}
\end{table}

Thus, for general losses involving nuisance $\ell$, it makes sense that the calibration algorithm $\calA_2$ should compute the following minimizer. We outline a general template for $\calA_2$ in Algorithm~\ref{alg:erm-cal} in Appendix~\ref{app:gen_cal}.
\[
\wh{\tau} \in \arg\min_{\tau \in \calF}\sum_{m = 1}^n\underbrace{\ell((\tau\circ\theta)(X_i), \wh{g}_i; Z_i)}_{=: \ell_i(\theta(X_i))},
\]
In the above, $Z_1, \dots, Z_n \sim P_Z$ is now the calibration sample and $\wh{g}_i$ denotes a nuisance estimate which generically may depend on the current sample $i$. Also in Appendix~\ref{app:gen_cal}, we also prove the convergence of a simple, three-way sample splitting algorithm based on uniform mass binning. and actually prove a finite sample convergence $L^2$ calibration error convergence guarantee for uniform mass binning as well. We do not include this algorithm in the main paper as the focus of the work is on presenting a general framework for performing causal calibration.

We can similarly define a version of cross calibration for conditionally orthogonal losses, which we present in Algorithm~\ref{alg:cross-cal-cond}.

\begin{algorithm}[ht]
   \caption{Cross Calibration for Conditionally Orthogonal Losses}
   \label{alg:cross-cal-cond}
\begin{algorithmic}[1]
   \Require Samples $\calD := Z_1, \dots, Z_{n} \sim P_Z$, nuisance alg.\ $\calA_1$, general loss calibration alg.\ $\calA_2$, estimator $\theta$, loss involving nuisance $\ell$.
   \State Split $\calI := [n]$ into $K$ equally-sized folds $\calI_1, \dots, \calI_K$.
   \State Define $\calD_k := \{Z_i : i \in \calI_k\}$ for all $k \in [K]$.
   \For{$k \in 1, \dots, K$}
   \State Estimate nuisances: $\wh{g}^{(-k)} \gets \calA_1(\theta, \calD\setminus \calD_k)$.
   \State Compute loss partial evaluations: $\ell_i := \ell(\cdot, \wh{g}^{(-k)}; Z_i)$ for $i \in \calI_k$. 
   \EndFor
   \State Run calibration: $\wh{\theta} \gets \calA_2(\theta, (X_m, \ell_m)_{m = 1}^{n})$
   \Ensure Calibrated estimator $\wh{\theta}$.
\end{algorithmic}
\end{algorithm}

We now focus our efforts on proving a convergence guarantee for Algorithm~\ref{alg:sample-split-cond}. We present a set of assumptions we need to prove our convergence guarantee.

\begin{ass}
\label{ass:alg_cond}
Let $\calA_1 : \calZ^\ast \times \Theta \rightarrow \calG$ be a nuisance estimation algorithm taking in an arbitrary number of points, and let $\calA_2 : \Theta \times (\calX \rightarrow \R)^\ast \rightarrow \Theta$ be a general loss calibration algorithm taking some initial estimator and a sequence of partially-evaluated loss functions $\ell_i : \calX \rightarrow \R$. 
\begin{enumerate}
    \item For any distribution $P_Z$ on $\calZ$, $Z_1, \dots, Z_n \sim_{i.i.d.} P_Z$, $\theta \in \Theta$, and failure probability $\delta_1 \in (0, 1)$, we have
    \[
    \err(\wh{g}, g_{\theta}; \gamma_{\theta}^\ast) \leq \rate_1(n, \delta_1, \theta; P_Z),
    \]
    where $\wh{g} = (\wh{\eta}, \wh{\zeta}) \gets \calA_1(Z_{1:n}, \theta)$, $g_{\theta} = (\eta_0, \zeta_{\theta})$, and $\rate_1$ is some rate function.
    \item For any distribution $P_Z$ on $\calZ$, $Z_1, \dots, Z_n \sim_{i.i.d.} P_Z$, $\theta \in \Theta$, $g \in \range(\calA_1) \subset \calG$ , and failure probability $\delta_2 \in (0, 1)$, we have
    \[
    \Cal(\wh{\theta}; g) \leq \rate_2(n, \delta_2, \ell; P_Z),
    \]
    where $\wh{\theta} \gets \calA_2(\theta, \{(X_m, \ell_m)\}_{m =1}^n)$, $\ell_m := \ell(\cdot, \wh{g}; Z_m)$, and $\rate_2$ is some rate function.
    \item With probability one, $\wh{\theta} = \tau \circ \theta$ for some injective mapping $\tau : \R \rightarrow \R$.
\end{enumerate}
\end{ass}

The first two assumptions are direct analogues of those made in Section~\ref{sec:universal}, giving the learner control over nuisance estimation and calibration rates. In general, $\zeta_\theta \neq \zeta_{\wh{\theta}}$ for arbitrary univariate mappings $\tau : \R \rightarrow \R$. This is a problem, as the learner will estimate the additional nuisance associated with the initial parameter, $\zeta_\theta$, but Theorem~\ref{thm:conditional} will be instantiated with respect to the nuisance $\zeta_{\wh{\theta}}$. The following lemma shows that injectivity of $\tau$ ensures $\zeta_\theta = \zeta_{\wh{\theta}}$.

\begin{restatable}{lemma}{lemlevelset}
\label{lem:level_set}
Suppose $\ell$ is conditionally orthogonal, and
suppose $\varphi_1, \varphi_2 : \calX \rightarrow \R$ have the same level sets, i.e.\ they satisfy $\{\varphi_1^{-1}(c) : c \in \range(\varphi_1)\} = \{\varphi_2^{-1}(c) : c \in \range(\varphi_2)\}$. \footnote{For a function $f : \calX \rightarrow \calY$ that is not necessarily injective, we let $f^{-1}(c) := \{x \in \calX : f(x) = c\}$.} Then, the calibration functions satisfy $\gamma_{\varphi_1}^\ast = \gamma_{\varphi_2}^\ast$. Additionally, without loss of generality, we can assume $\zeta_{\varphi_1} \equiv \zeta_{\varphi_2}$.
\end{restatable}

Calibration algorithms that learn an injective post-processing mapping, such as Platt scaling and linear calibration, will preserve level sets. 
For an algorithm like isotonic calibration, one can either (a) estimate $\zeta_{\theta}$ and hope $\zeta_\theta \approx \zeta_{\wh{\theta}}$ or (b) if $\wh{\tau} \in \{f : \R \rightarrow \R \mid f \text{ is monotonic}\}$ learned via isotonic regression is not strictly increasing, one can release $\wh{\theta} = (\wh{\tau} + \xi)\circ \theta$ where $\xi$ is \textit{any} strictly increasing map.
For algorithms such as histogram binning and uniform mass binning, one can first learn then the level sets of $\wh{\theta}$, which are just based on the quantiles of the predictions of the initial estimator $\theta(X)$. Then, these level sets entirely specify the target nuisance $\zeta_{\wh{\theta}}$ to estimate.  We provide a version of Algorithm~\ref{alg:sample-split-cond} in Appendix~\ref{app:gen_cal} that does this. We now provide a convergence guarantee for Algorithm~\ref{alg:sample-split-cond}.

\begin{restatable}{theorem}{thmalgconditional}
\label{thm:alg_cond}
Suppose $\ell : \R \times \calG \times \calZ \rightarrow \R$ is an arbitrary, conditionally orthogonal loss 
Let $\calA_1$ and $\calA_2$ satisfy Assumption~\ref{ass:alg_cond}, and suppose $\wh{\theta} = \wh{\tau} \circ \theta$ where $\wh{\tau}$ is almost surely injective. Then, with probability at least $1 - \delta_1 - \delta_2$, the output $\wh{\theta}$ of Algorithm~\ref{alg:sample-split-universal} run on a calibration dataset of $2n$ i.i.d.\ data points $Z_{1}, \dots, Z_{2n} \sim P$ satisfies
\[
\Cal(\wh{\theta}, \eta_0) \leq \frac{1}{2}\rate_1(n, \delta_1, \theta; P_Z) + \rate_2(n, \delta_2, \ell; P_Z).
\]
\end{restatable}

\section{Experiments}
\label{sec:experiments}
We now evaluate the performance of our cross-calibration algorithms, namely Algorithms~\ref{alg:cross-cal-universal} and \ref{alg:cross-cal-cond}. We consider two settings. First, we examine the viability of of our calibration algorithm for universally orthogonal losses using observational data. In particular, we show how Algorithm~\ref{alg:cross-cal-universal} can be used to calibrate estimates of the CATE of 401(k) eligibility and the conditional LATE of 401(k) participation on an individual's net financial assets. 
Second, we measure the ability of Algorithm~\ref{alg:cross-cal-cond} to calibrate estimates of conditional quantiles under treatment on synthetic data. We examine the performance of our algorithm for several quantile both in terms of calibration error and average loss.

\subsection{Effects of 401(k) Participation/Eligibility on Financial Assets}

First, we consider the task of constructing and calibrating estimates of the heterogeneous effect of 401(k) eligibility/participation on an individual's net financial assets. To do this, we use the 401(k) dataset leveraged in \citet{chernozhukov2004effects}, \citet{kallus2019localized}, and  \citet{chernozhukov2024applied}. Since \citet{chernozhukov2004effects} argue that eligibility for 401(k) satisfies conditional ignorability given a sufficiently rich set of features\footnote{The entire set of features are as follows: \texttt{[`age', `inc', `fsize', `educ', `db', `marr', `male', `twoearn', `pira', `nohs', `hs', `smcol', `col', `hown']}}, we aim to measure the CATE of eligibility on net financial assets. This ignorability is not known to be satisfied for 401(k) participation, and thus we instead aim to measure to conditional LATE of 401(k) participation on net financial assets with eligibility serving as an instrument. For each parameter (either CATE or conditional LATE), we randomly split the dataset into three folds of uneven size: we use 60\% of the data to construct the initial parameter estimate, 25\% of the data to perform calibration, and reserve 15\% of the data as a test set. 

\paragraph{Model Training and Calibration:}

To fit the initial model CATE/conditional LATE model, we split the training data randomly into $K = 5$ evenly-sized folds. We use cross-fitting to construct appropriate pseudo-outcomes, i.e.\ for each $k \in [5]$ we use data in all but the $k$th fold to estimate nuisances and then appropriately use these estimates to transform observations in the $k$th fold. In the case of the CATE, we produce estimates $\wh{\pi}^{(-k)}$, $\wh{\mu}^{(-k)}$ of the propensity score $\pi_0(x) := \P(A = 1 \mid X = x)$ and expected outcome mapping $\mu_0(x) := \E[Y \mid X = x, A = a]$ using gradient boosted decision and regression tress, respectively. We then produce pseudo-outcomes on the $k$th fold, per Equation~\eqref{eq:pseudo_CATE}, and use gradient-boosted regression trees to regress these pseudo-outcomes onto covariates. 

In the case of the conditional LATE, as the instrument policy $\zeta_0(x) := \P(Z = 1 \mid X =x)$ is not assumed to be known, instead of using the universally orthogonal loss discussed in Equation~\eqref{eq:pseudo_LATE}, we instead use the following loss detailed in \citet{syrgkanis2019machine} and \citet{lan2023causal}:
\[
\ell(\theta, (\mu, \pi, \zeta); W) := \frac{1}{2}\Big(\theta(x) - \quad \underbrace{\tau(x) - \frac{(\wt{y} - 
\tau(x)\wt{a})\wt{Z}}{\zeta(x)} }_{=:\chi_{\LATE}'(g; W)}\quad\Big)^2
\]
where $\tau(x) := \frac{\mu(1, x) - \mu(0, x)}{\zeta(x)}$, $\wt{y} := y - \mu(a, x)$, $\wt{Z} := Z - \zeta(x)$, and $\wt{a} := a - \pi(x)$, where $\mu_0$ and $\pi_0$ are as defined above. We estimate all nuisances and construct pseudo-outcomes $\chi'_{\LATE}$ as in the case of the CATE, once again using either gradient-boosted regression or decision trees based on appropriateness. Again, we regress pseudo-outcomes onto covariates using gradient-boosted regression trees to obtain initial parameter estimates.

After initial models are trained, we run Algorithm~\ref{alg:cross-cal-universal} (cross-calibration) on the 25\% of the data reserved for calibration, once again using $K = 5$ folds. We estimate all nuisances again by either using gradient-boosted decision or regression trees. We perform calibration using three different algorithms: isotonic calibration, histogram binning with $B = 20$ buckets, and linear calibration, which performs simple linear regression with intercept of the constructed pseudo-outcomes onto the initial model predictions.

\paragraph{Comparing Calibration Error in Quartiles} We assess calibration of both the pre- and post-calibrated models by approximating the actual target treatment effect with the quartiles of the models' predictions. We let $\wh{\theta}$ denote either the pre or post calibrated model (either CATE or conditional LATE). Re-using the calibration dataset (25\% of the data), we compute the order statistics of model predictions $\wh{\theta}(X)_{(1)}, \dots, \wh{\theta}(X)_{(N)}$. We then define four buckets based on the quartiles of the above order statistics, i.e.
\[
B_i = \left\{x \in \calX : \wh{\theta}(x) \in \Big(\wh{\theta}(X)_{(\lfloor iN/4\rfloor)}, \wh{\theta}(X)_{(\lfloor (i + 1)N/4\rfloor)}\Big] \right\} \qquad \text{for } i = 0, 1, 2, 3
\]
where we assume $\wh{\theta}(X)_{(0)} = -\infty$ and $\wh{\theta}(X)_{(N)} = \infty$. Next, we use cross-fitting with $K = 5$ folds to transform the 15\% of the data reserved for testing into pseudo-outcomes $\chi_m$.\footnote{This is done in the same manner as discussed in the previous subsection} We assign each transformed sample $(X_m, \chi_m)$ in the test set to an appropriate bin based on the predicted value of $\wh{\theta}(X_m)$, and average the pseudo outcomes falling into each bin. We then approximate the $L^2$ calibration error, which is computed as: $\wh{\Cal}(\wh{\theta}, g_0) = \frac{1}{4}\sum_{i = 1}^4 (\overline{\chi}_i - \overline{\theta}_i)^2$, where $\overline{\theta}_i$ denotes the average of the $\wh{\theta}(X_m)$ falling into bin $i$ and $\wh{\chi}_i$ denotes the average pseudo-outcome $\chi_m$ in bin $i$.

\begin{figure}
    \centering
    \begin{subfigure}{0.4\textwidth}
    \centering
    \includegraphics[width=\linewidth]{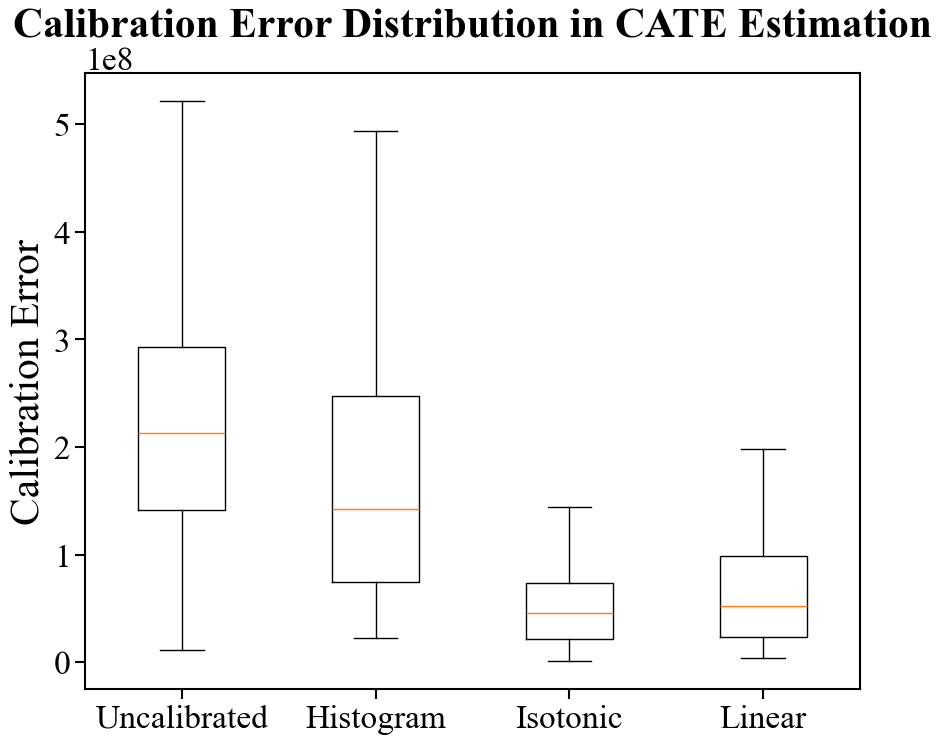}
    \caption{Calibration Error for CATE}
    \label{fig:box:CATE}
    \end{subfigure}
    \begin{subfigure}{0.4\textwidth}
    \centering
    \includegraphics[width=\linewidth]{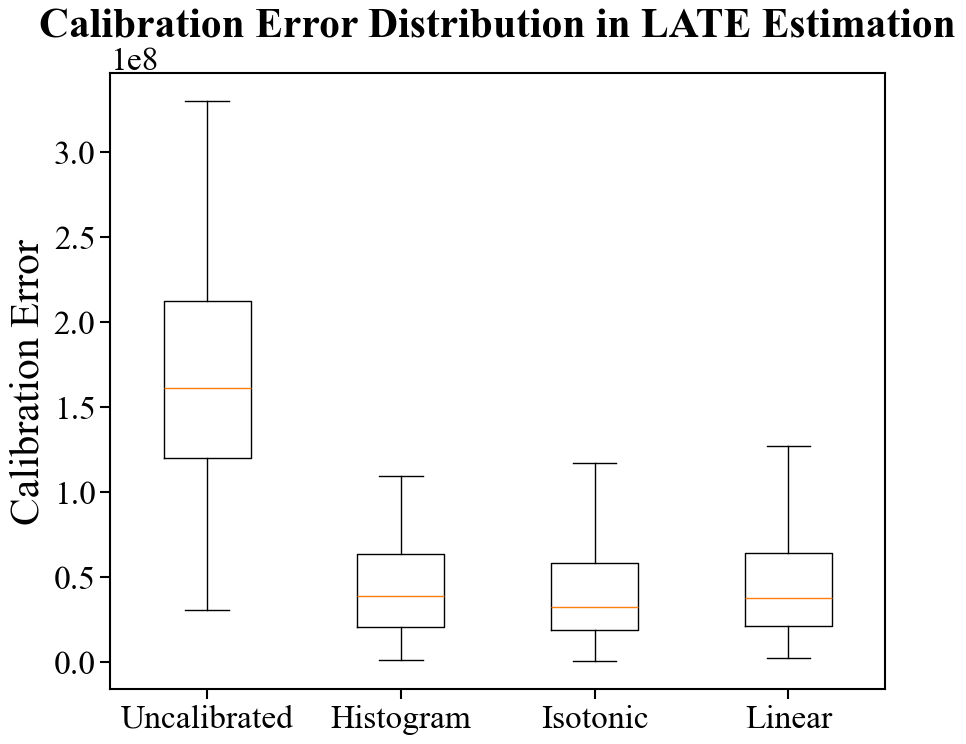}
    \caption{Calibration Error for LATE}
    \label{fig:box:LATE}
    \end{subfigure}
    \caption{}
    \label{fig:box}
\end{figure}

We repeat the above experiment over $M = 100$ random splits of the initial 401(k) dataset. Figure~\ref{fig:box} shows in a box and whisker plot the empirical distribution of the $L^2$ calibration errors over these $M$ runs for the base (uncalibrated) model, and the model calibrated using each of thee aforementioned calibration algorithms. Each box is centered at the median calibration error $m$, the bottom of the box is given by the first quartile of calibration error $Q1$, and the top of the box is given by the third quartile $Q3$ of calibration error. The top (bottom) of the whiskers are given by the maximum (resp. minimum) of the observations falling within $m \pm 1.5\times(Q3 - Q1)$. 

From Figure~\ref{fig:box}, we see that all calibration algorithms result in noticeably smaller calibration error. In particular in the setting of the CATE of 401(k) eligibility on net financial assets, the third quartile for the calibration error using isotonic calibration and linear calibration falls below the first quartile of calibration error for the uncalibrated model. Likewise, in the LATE model, which aims to measure the effect of 401(k) participation on net financial assets, the third quartile of calibration error under histogram binning and linear calibration falls below the first quartile of calibration error for the uncalibrated model. This indicates that our calibration algorithms have a significant impact on the calibration error of the produced models.

\subsection{Calibrating Estimates of Conditional Quantiles Under Treatment}

We now show how Algorithm~\ref{alg:cross-cal-cond} (cross-calibration for conditionally orthogonal loss functions) can be used to calibrate estimates of the conditional $Q$th quantile under treatment. In particular, we study the the impact of sample size and chosen quantile have on both $L^2$ calibration error and average loss.

For this experiment, we leverage synthetically generated data. At the start of the experiment, we generate two slope vector $\beta^Y, \beta^\pi \in \R^{100}$ where $\beta^Y_{1:20}, \beta^\pi_{1:20},  \sim \calN(0, I_{20})$ are i.i.d.\ and $\beta^Y_i = \beta^\pi_i = 0$ for all $i > 20$. Then, for $N \in [500, 1000, 1500, 2000, 2500, 3000]$, we generate i.i.d.\ covariates $X_1, \dots, X_N \sim \calN(0, I_{100})$ and treatments $A_1, \dots, A_N$ where $A_i \sim \mathrm{Bern}(p_i)$ and 
\[
p_i := \max\left\{C, \min\left\{\frac{1}{1 + \exp(\langle \beta^\pi, X_i\rangle)}, 1 - C\right\}\right\}.
\]
In the above, $C = 0.05$ is a fixed clipping parameter. Finally, we generate potential outcomes under treatment and control as $Y_i(1) = Y_i(0) = \langle \beta^Y, X_i\rangle + \epsilon_i$, where $\epsilon_1, \dots, \epsilon_N$ are i.i.d.\ standard normal noise variables. We note that, because we are only interested in the conditional quantile under treatment, we make potential equivalent under both treatment and control for simplicity.

Given data as generated above, we then train an initial model $\wh{\theta}$ using gradient-boosted regression trees. In more detail, in training these regression trees, we leverage the loss $\wt{\ell}_{\QTE}(\theta, p; z) := a \cdot p(x) (y - \theta(x))(Q -\mathbbm{1}_{y \leq \theta(x)})$, where we determine the nuisances using $K = 5$ fold cross-fitting, i.e.\ where we estimate the propensity $\pi_0(x) := \P(A = 1 \mid X = x) \equiv p_0(x)^{-1}$ using gradient-boosted decision trees. We note that this loss is Neyman orthogonal conditional on covariates $X$, and thus using this loss allows for the efficient estimate of the quantile under treatment. However, this loss is not conditionally orthogonal, and thus is a poor fit for performing calibration.

We then use an additional $N$ samples generated in the same manner as above to perform calibration using Algorithm~\ref{alg:cross-cal-cond} and loss function $\ell_{\QTE}(\theta, (p, f); z)$ (outlined in Equation~\eqref{eq:qte_loss}). We estimate the inverse propensity $p_0(x)$ by using gradient-boosted decision trees to estimate the propensity $\pi_0(x) = p_0(x)^{-1}$. We estimate the additional CDF-like nuisance $f_{\wh{\theta}}(x) := \P(Y(1) \leq \gamma_{\wh{\theta}}^\ast(X) \mid X = x)$ by instead estimating $\P(Y(1) \leq \wh{\theta}(X) \mid X = x)$, again using gradient-boosted decision trees. Heuristically, we are hoping that $\wh{\theta}$ is a ``reasonable rough estimate'' for the calibration function $\gamma_{\wh{\theta}}^\ast$. In both settings, we perform calibration using linear calibration. After we estimate nuisance, we run the cross-calibration algorithm by using linear calibration.
We repeat the entire above process $M = 50$ times for three quantile values $Q \in \{0.5, 0.75, 0.9\}$, and plot both the empirical $L^2$ calibration errors and average losses below (measuring loss according to $\wt{\ell}_{\QTE}(\theta, p_0; z)$, i.e.\ the loss evaluated at the true nuisances).

\paragraph{Results:}

Figure~\ref{fig:quant} displays the results of the experimental procedure outlined above. Displayed in the left-hand column are plots demonstrating the empirical $L^2$ calibration at various sample sizes averaged over the $M = 50$ runs. Likewise, in the right-hand column, the sample loss is displayed averaged again over $M$ runs. We include point-wise valid 95\% confidence intervals for all plots. 

Regardless of the sample size $N$ and chosen quantile $Q$, we see a significant decrease in the $L^2$ calibration error. This shows that not only does Algorithm~\ref{alg:cross-cal-cond} exhibit favorable theoretical convergence guarantees, but that it also offer strong performance in practice. Typically, one hopes calibration algorithms enjoy a ``do no harm'' property, i.e.\ that calibrating a parameter estimate will not significantly increase loss. While we do not formally prove this, the right-hand column of Figure~\ref{fig:quant} demonstrates that calibrating typically decreases loss, as desired. Moreover, the loss obtained by using $N$ samples to construct an initial estimate and $N$ to perform calibration is comparable to the loss had $2N$ samples directly been used to estimate the unknown conditional parameter. This suggests that reserving some samples for calibration yields significantly lower calibration error without adversely affecting performance.

\begin{figure}
    \centering
    \begin{subfigure}{0.4\textwidth}
    \centering
    \includegraphics[width=\linewidth]{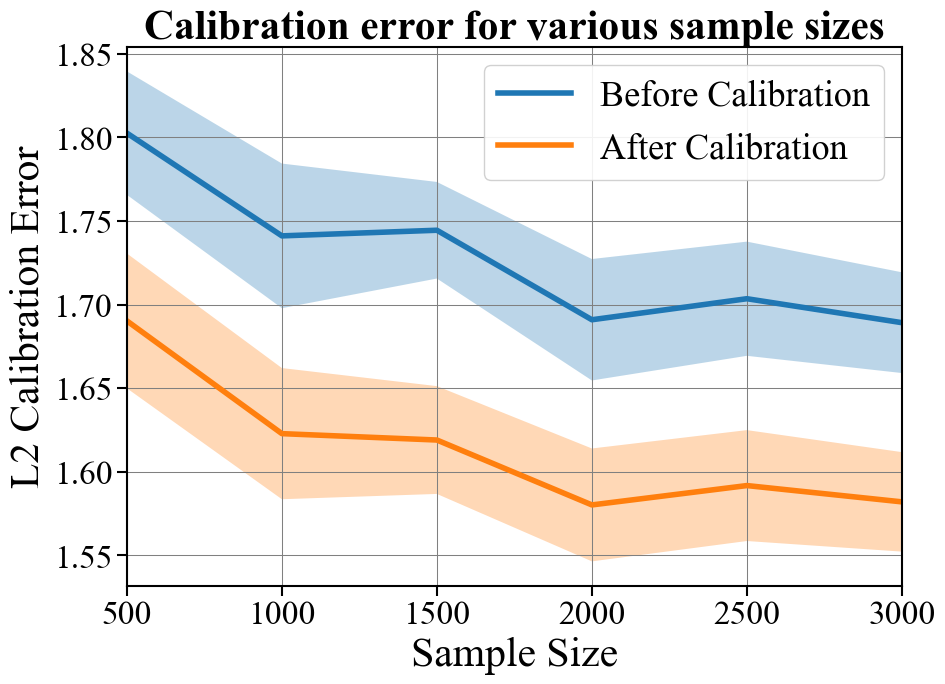}
    \caption{Calibration Error for $Q = 0.6$}
    \label{fig:quant:error_0.6}
    \end{subfigure}
    \begin{subfigure}{0.4\textwidth}
    \centering
    \includegraphics[width=\linewidth]{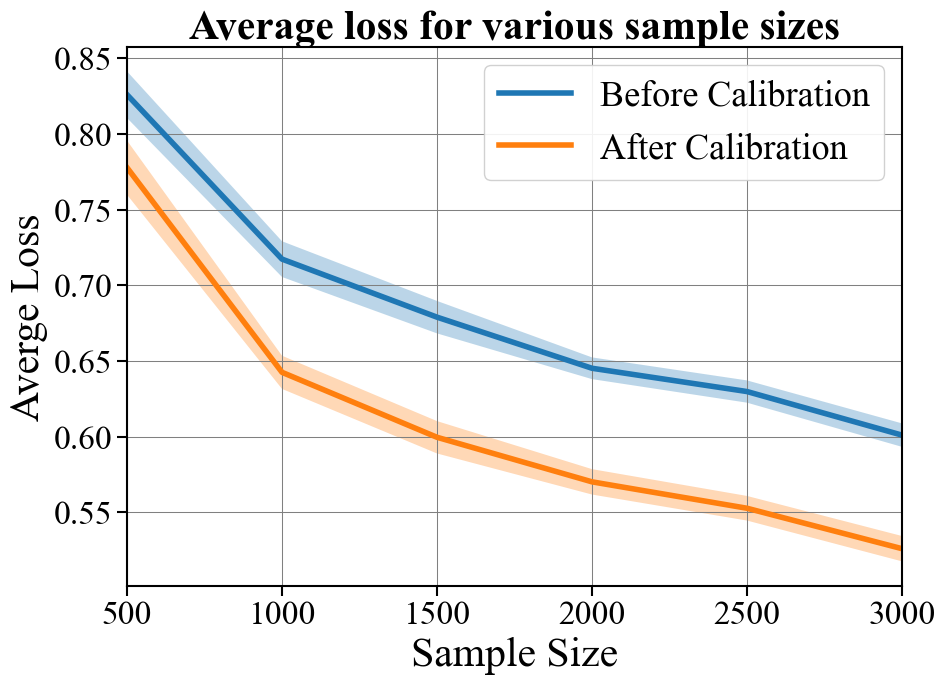}
    \caption{Average Loss for $Q = 0.6$}
    \label{fig:quant:loss_0.6}
    \end{subfigure}
    \begin{subfigure}{0.4\textwidth}
    \centering
    \includegraphics[width=\linewidth]{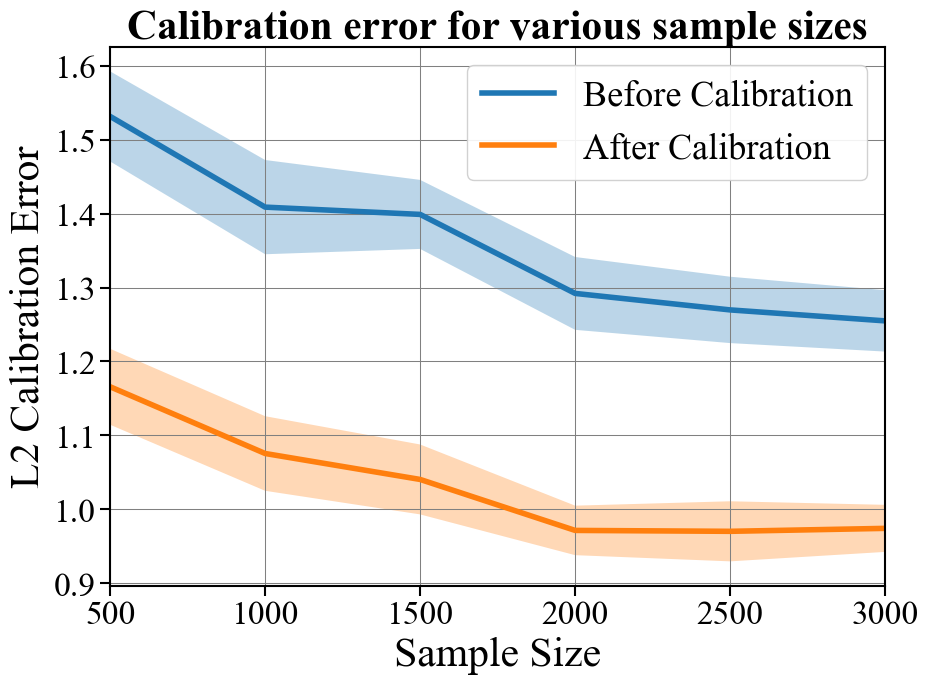}
    \caption{Calibration Error for $Q = 0.75$}
    \label{fig:quant:error_0.75}
    \end{subfigure}
    \begin{subfigure}{0.4\textwidth}
    \centering
    \includegraphics[width=\linewidth]{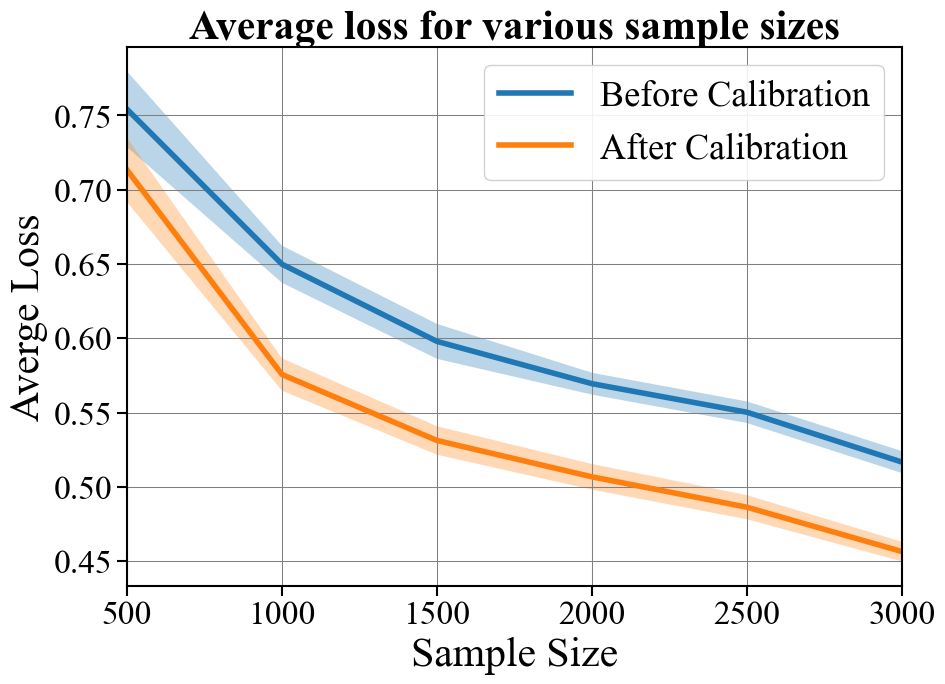}
    \caption{Average Loss for $Q = 0.75$}
    \label{fig:quant:loss=0.75}
    \end{subfigure}
    \begin{subfigure}{0.4\textwidth}
    \centering
    \includegraphics[width=\linewidth]{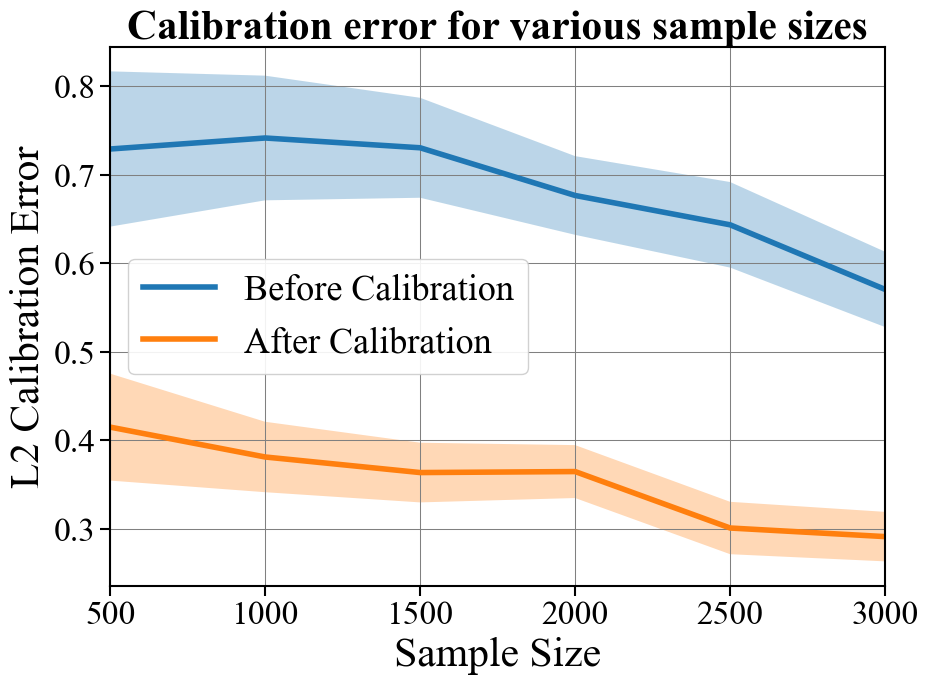}
    \caption{Calibration Error for $Q = 0.9$}
    \label{fig:quant:error_0.9}
    \end{subfigure}
    \begin{subfigure}{0.4\textwidth}
    \centering
    \includegraphics[width=\linewidth]{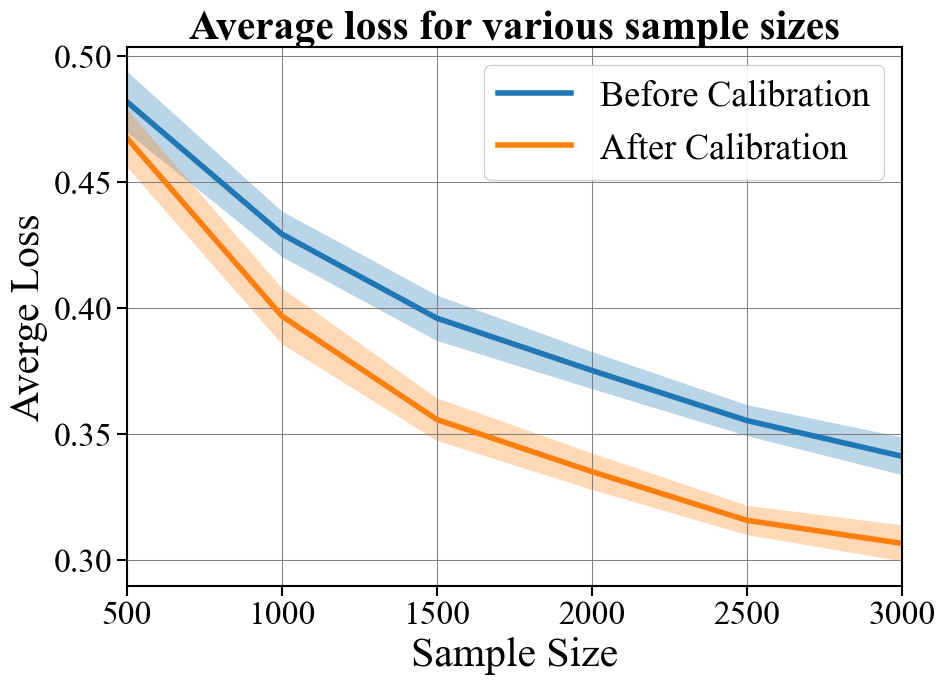}
    \caption{Average Loss for $Q = 0.9$}
    \label{fig:quant:loss_0.9}
    \end{subfigure}
    \caption{We plot the performance of Algorithm~\ref{alg:cross-cal-cond} in calibrating estimates of conditional quantiles under treatment using linear calibration. We display both the sample $L^2$ calibration error and the average loss for $N \in \{500, 1000, 1500, 2000, 2500, 3000\}$ and $Q \in \{0.6, 0.75, 0.9\}$ (where an additional $N$ samples are used for calibration). We also display corresponding 95\% pointwise-valid confidence intervals. Cross-calibration not only decreases calibration error (as expected), but also decreases loss.}
    \label{fig:quant}
\end{figure}

\section{Conclusion}
\label{sec:conc}

In this work, we constructed a framework for calibrating general estimates of heterogeneous causal effects with respect to nuisance-dependent losses. By leveraging variants of Neyman orthogonality, we were able to bound the $L^2$ calibration error of $\theta$ by two decoupled terms. One term, roughly, represented the error in nuisance estimation, while the other term represented the $L^2$ calibration error of $\theta$ in a world where the learned nuisances were true. These bounds suggested natural sample-splitting and cross-calibration algorithms, for which we proved high-probability convergence guarantees. Our algorithms also admitted ``do no harm'' style guarantees. We empirically evaluated our algorithms in Section~\ref{sec:experiments}, in which we considered both observational and synthetic data.

While our provided contributions are quite general, there are still interesting directions for future work. First, in our work, we only measure the convergence of our algorithms via the $L^2$-calibration error. Depending on the situation, other notions of calibration error may be more appropriate. For instance, \citet{gupta2021distribution} analyze the convergence of histogram/uniform mass binning in terms of $L^\infty$ calibration error. Likewise, \citet{globus2023multicalibration} study $L^1$ muli-calibration error. We leave it as interesting future work to extend our results to the general setting of measuring $L^p$ calibration error. Further, it would be interesting to study how the calibration of estimates of heterogeneous causal effects impacts utility in decision making tasks.

\section{Acknowledgments}

VS is supported by NSF Award IIS 2337916.

\newpage
\bibliography{bib.bib}{}
\bibliographystyle{plainnat}

\appendix

\section{Calibration Error Decomposition Proofs}
\label{app:decomp}

In this appendix, we prove the main error decompositions from Sections~\ref{sec:universal} an d\ref{sec:conditional}. These results provide two-term, decoupled bounds on the $L^2$ calibration error of an arbitrary, fixed parameter $\theta : \calX \rightarrow \R$ in terms of $L^2$ calibration error assuming the learned nuisances were correct, and a term measuring the distance between the learned nuisances and the true, unknown nuisance parameters. 

\subsection{Universally Orthogonality}
\label{app:decomp:universal}
We start by proving Theorem~\ref{thm:universal}, which provides the claimed decoupled bound under the assumption that $\ell$ is universally orthogonal (Definition~\ref{def:universal}).
\thmuniversal*
\begin{proof}[Proof of Theorem~\ref{thm:universal}]
We start by adding a useful form of zero to the integrand, which yields
\begin{align*}
\Cal(\theta, g_0)^2 &= \E\left(\E\left[\partial \ell(\theta, g_0; Z) \mid \theta(X) \right]^2\right) \\
&= \underbrace{\E\left(\E\left[\partial \ell (\theta, g_0; Z) \mid \theta(X) \right] \cdot \left\{\E\left[\partial \ell (\theta, g_0; Z) \mid \theta(X)\right] - \E\left[\partial \ell(\theta, g ; Z) \mid \theta(X)\right]\right\}\right)}_{T_1} \\
&\qquad + \underbrace{\E\left(\E\left[\partial \ell(\theta, g_0; Z) \mid \theta(X)\right] \cdot \E\left[\partial \ell(\theta, g; Z) \mid \theta(X)\right]\right)}_{T_2}.
\end{align*}

We bound $T_1$ and $T_2$ separately. As a first step in bounding $T_1$, note that by a second order Taylor expansion with respect to the nuisance estimate $g \in \calG$ with Lagrange remainder, we have
\begin{align*}
&-\E\left[\partial \ell (\theta, g_0; Z) \mid \theta(X)\right] + \E\left[\partial \ell(\theta, g ; Z) \mid \theta(X)\right] \\
&\qquad = D_{g}\E\left[\partial \ell(\theta, g_0 ; Z) \mid \theta(X)\right](g - g_0) + \frac{1}{2}D_{g}^2\E\left[\partial \ell(\theta, \wb{g}; Z) \mid \theta(X)\right](g - g_0, g - g_0) \\
&\qquad=\frac{1}{2}D_{g}^2\E\left[\partial \ell(\theta, \wb{g}; Z) \mid \theta(X)\right](g - g_0, g - g_0).
\end{align*}
In the above, $\wb{g} \in [g_0, g]$, and the first-order derivative (with respect to $g$) vanishes due to the assumption of Definition~\ref{def:universal}. This is because we have Taylor expanded around the true, unknown nuisance $g_0 = g_0$.

With this, we can apply the Cauchy-Schwarz inequality, which furnishes
\begin{align*}
T_1 &\leq \frac{1}{2}\sqrt{\E\left(\E\left[\partial \ell(\theta, g_0 ; Z) \mid \theta(X)\right]^2\right)}\sqrt{\E\left(\left\{D_{g}^2\E\left[\partial \ell(\theta, \wb{g}; Z) \mid \theta(X)\right](g - g_0, g - g_0)\right\}^2\right)}\\
&\leq \frac{1}{2}\sqrt{\E\left(\E\left[\partial \ell(\theta, g_0 ; Z) \mid \theta(X)\right]^2\right)}\sqrt{\E\left(\left\{D_{g}^2\E\left[\partial \ell(\theta, \wb{g}; Z) \mid X\right](g - g_0, g - g_0)\right\}^2\right)}\\
&\leq \frac{1}{2}\Cal(\theta, g_0)\cdot \err(g, g_0; \theta)
\end{align*}
In the second line, we apply Jensen's inequality inside the conditional expectation.

Bounding $T_2$ is more straightforward. Applying the Cauchy-Schwarz inequality, we have:
\begin{align*}
T_2 &\leq \sqrt{\E\left(\E\left[\partial \ell(\theta, g_0; Z) \mid \theta(X)\right]^2\right)}\sqrt{\E\left(\E\left[\partial \ell(\theta, g; Z) \mid \theta(X)\right]^2\right)} \\
&= \Cal(\theta, g_0)\cdot\Cal(\theta, g)
\end{align*}

This line of reasoning, in total, yields that
\[
\Cal(\theta, g_0)^2 \leq \frac{1}{2}\Cal(\theta, g_0) \cdot \err(g, g_0; \theta) + \Cal(\theta, g_0) \cdot \Cal(\theta, g).
\]

Dividing through by $\Cal(\theta, g_0)$ yields the claimed bound.
\end{proof}

\subsection{Conditional Orthogonality}
\label{app:decomp:conditional}
We now turn to proving the second error bound, which holds in the case that $\ell$ satisfies the weaker assumption of conditional orthogonality (Definition~\ref{def:conditional}). We generically write $g = (\eta, \zeta)$ for a fixed pair of nuisance parameters. Further, for any fixed post-processing function $\varphi : \calX \rightarrow \calX'$, we write  $g_\varphi = (\eta_0, \zeta_\varphi)$ where $\zeta_\varphi$ is a nuisance parameter ensuring the vanishing cross-derivative condition. To prove Theorem~\ref{thm:conditional}, we will need two technical lemmas. In what follows, we remind the reader of the definition of the \textit{calibration function} $\gamma_\varphi : \calX \times \calG \rightarrow \R$ denote the calibration function under the orthogonalized loss $\ell$, i.e.\ $\gamma_\varphi$ is specified by
\[
\gamma_\theta(x; g) := \arg\min_\nu \E[\ell(\theta, g; Z) \mid \theta(X) = \theta(x)],
\]
We recall the identity $\gamma_\varphi^\ast \equiv \gamma_\varphi(\cdot; (\eta_0, \zeta))$ for any estimate $\zeta$ of the second nuisance parameter, which will be useful in the sequel.

The first lemma we prove measures the distance (in terms of the $L^2(P_X)$ norm) between the true calibration $\gamma_\theta^\ast$ and the calibration function under any other nuisance pair $g = (\eta, \zeta)$, $\gamma_\theta(\cdot; g)$. Here, $\theta : \calX \rightarrow \R$ should be viewed as some arbitrary estimator. We can bound this distance in terms of the complicated looking error term, which was first introduced in Theorem~\ref{thm:universal}. This term actually simplifies rather nicely, as was seen in the prequel when we computed the quantity for the task of calibrating estimates of conditional $Q$th quantiles under treatment.

\begin{lemma}
\label{lem:error_cal_cal}
Let $\theta : \calX \rightarrow \R$ be an arbitrary function, assume $\ell$ is conditionally orthogonal, and let $g_\theta = (\eta_0, \zeta_\theta)$ be a pair of nuisance functions guaranteeing the vanishing cross-derivative condition in Definition~\ref{def:conditional}. Let $g = (\eta, \zeta)$ be some fixed pair of nuisance functions. Then, assuming the base loss $\ell$ satisfies $\alpha$-strong convexity (Assumption~\ref{ass:strong_conv}), we have
\[
\|\gamma_\theta^\ast - \gamma_\theta(\cdot; g)\|_{L^2(P_X)} \leq \frac{1}{2\alpha}\err(g, g_\theta; \gamma_\theta^\ast),
\]
where we define 
$\err(g, h; \varphi) := \sup_{f \in [g, h]}\sqrt{\E\left(\left\{D_g^2\E\left[\partial \ell(\varphi, f; Z) \mid X\right](g - h, g - h)\right\}^2\right)}$.
\end{lemma}

\begin{proof}
First, strong convexity (Assumption~\ref{ass:strong_conv}) alongside equivalent conditions for strong convexity (namely that $\alpha(x - y)^2 \leq (\partial f(x) - \partial f(y))(x - y)$) yields:
\begin{align*}
&\alpha\left(\gamma_{\theta}(X; g_\theta) - \gamma_{\theta}(X; g)\right)^2 \\
&\qquad \leq \left(\E\left[\partial \ell(\gamma_{\theta}(\cdot; g_\theta), g; Z) \mid \theta(X)\right] - \E\left[\partial \ell(\gamma_\theta(\cdot; g), g; Z) \mid \theta(X)\right]\right)\left(\gamma_{\theta}(X; g_\theta) - \gamma_{\theta}(X; g)\right) \\
&\qquad= \E\left[\partial \ell(\gamma_{\theta}(\cdot; g_\theta), g; Z) \mid \theta(X)\right]\left(\gamma_{\theta}(X; g_\theta) - \gamma_{\theta}(X; g)\right)
\end{align*}

In the above, the equality on the third line follows from the definition of $\gamma_{\theta}(\cdot; g)$, as first order optimality conditions on $\gamma_{\theta}(x; g) = \arg\min_{\nu \in \R}\E\left[\ell(\nu, g; Z) \mid \theta(X)\right]$ imply 
$\E\left[\partial \ell(\gamma_\theta(\cdot; g), g; Z) \mid \theta(X)\right] = 0$. 

Rearranging the above inequality and taking absolute values yields 
\[
\alpha \left|\gamma_{\theta}(X; g_\theta) - \gamma_{\theta}(X; g)\right| \leq \left|\E\left[\partial \ell(\gamma_{\theta}(\cdot; g_\theta), g; Z) \mid \theta(X)\right]\right|.
\]

Next, observe that from the condition $\E\left[\partial \ell(\gamma_{\theta}(\cdot; g_\theta), g_\theta; Z) \mid \theta(X)\right] = 0$ alongside a second order Taylor expansion (with respect to nuisance pairs
$g$) with Lagrange form remainder plus conditional orthogonality (Definition~\ref{def:conditional}), we have
\begin{align*}
&\E\left[\partial \ell(\gamma_{\theta}(\cdot; g_\theta), g; Z) \mid \theta(X)\right] \\
&\qquad = \E\left[\partial \ell(\gamma_{\theta}(\cdot; g_\theta), g; Z) \mid \theta(X)\right] - \E\left[\partial \ell(\gamma_{\theta}(\cdot; g_\theta), g_\theta;  Z) \mid \theta(X)\right] \\
&\qquad = D_{g}\E\left[\partial \ell(\gamma_{\theta}(\cdot; g_\theta), g_\theta; Z) \mid \theta(X)\right](g - g_\theta) + \frac{1}{2}D_{g}^2\E\left[\partial \ell(\gamma_{\theta}(\cdot; g_\theta), \wb{g}; Z) \mid \theta(X)\right](g - g_\theta, g - g_\theta)\\
&\qquad = \frac{1}{2}D_{g}^2\E\left[\partial \ell(\gamma_{\theta}(\cdot; g_\theta), \wb{g}; Z) \mid \theta(X)\right](g - g_\theta, g - g_\theta) \\
&\qquad = \frac{1}{2}\E\left(D_{g}^2 \E\left[\partial \ell(\gamma_{\theta}(\cdot; g_\theta), \wb{g}; Z) \mid X\right](g - g_\theta, g - g_\theta) \mid \theta(X)\right),
\end{align*}

where $\wb{g} \in [g, g_\theta]$ (here, $\wb{g} \in [g, g_\theta]$ indicates $\wb{g} = \lambda g + (1 - \lambda) g_\theta$ for some $\lambda \in [0, 1]$). Consequently, we have
\begin{align*}
&\|\gamma_{\theta}(X; g_\theta) - \gamma_{\theta}(X; g)\|_{L^2(P_X)} \\
&\qquad = \left(\int_{\calX}\left|\gamma_{\theta}(x; g_\theta) - \gamma_{\theta}(x; g)\right|^2P_X(dx)\right)^{1/2} \\
&\qquad \leq \frac{1}{\alpha}\left(\int_{\calX}\E\left[\partial \ell(\gamma_{\theta}^\ast, g; Z) \mid \theta(X) = \theta(x)\right]^2 P_X(dx)\right)^{1/2} \\
&\qquad \leq \frac{1}{2\alpha}\left(\int_{\calX} \E\left(D_{g}^2 \E\left[\partial \ell(\gamma_{\theta}(\cdot; g_\theta), \wb{g}; Z) \mid X\right](g - g_\theta, g - g_\theta) \mid \theta(X) = \theta(x)\right)^2P_X(dx)\right)^{1/2} \\
&\qquad \leq \frac{1}{2\alpha}\left(\int_{\calX}\left\{D_{g}^2 \E\left[\partial \ell(\gamma_{\theta}(\cdot; g_\theta), \wb{g}; Z) \mid X = x\right](g - g_\theta, g - g_\theta)\right\}^2P_X(dx)\right)^{1/2} \\
&\qquad \leq \frac{1}{2\alpha}\err(g, g_\theta; \gamma_\theta(\cdot; g_\theta)).
\end{align*}
Noting the identity $\gamma_\theta(\cdot; g_\theta) \equiv \gamma_\theta^\ast$ proves the claimed result.
\end{proof}

The second lemma we prove bounds the $L^2(P_X)$ distance between the parameter estimate $\theta$ and the calibration function $g_\theta(\cdot; g)$ under a fixed pair of nuisances $g = (\eta, \zeta)$  in terms of the calibration error.

\begin{lemma}
\label{lem:error_est_cal}
Let $\theta : \calX \rightarrow \R$ be a fixed estimator, $g = (\eta, \zeta) \in \calG$ an arbitrary, fixed nuisance pair, and $\gamma_\theta : \calX \times \calG \rightarrow \R$ the calibration function associated with $\theta$. Assume $\ell$ is $\alpha$-strongly convex (Assumption~\ref{ass:strong_conv}). We have
\[
\left\|\theta - \gamma_\theta(\cdot; g)\right\|_{L^2(P_X)} \leq \frac{1}{\alpha}\Cal(\theta, g).
\]

\end{lemma}

\begin{proof}
    First, observe that from strong convexity (as in the proof of the above lemma), we have
    \begin{align*}
    &\alpha\left(\theta(X)  - \gamma_{\theta}(X; g)\right)^2 \\
    &\qquad \leq \left(\E\left[\partial \ell (\theta, g; Z) \mid \theta(X)\right] - \E\left[\partial \ell(\gamma_{\theta}(\cdot, g), g; Z) \mid \theta(X)\right]\right)\left(\theta(X) - \gamma_\theta(X; g)\right) \\
    &\qquad = \E\left[\partial \ell (\theta, g; Z) \mid \theta(X)\right]\left(\theta(X) - \gamma_\theta(X; g)\right).
    \end{align*}

    Thus, dividing through and taking the absolute value yields:
    \[
    \left|\theta(X) - \gamma_\theta(X; g)\right| \leq \frac{1}{\alpha}\left|\E\left[\partial \ell(\theta, g; Z) \mid \theta(X)\right]\right|.
    \]

    We now integrate to get the desired result. In particular, we have that

    \begin{align*}
    \|\theta - \gamma_\theta(\cdot; g)\|_{L^2(P_X)} &= \left(\int_{\calX}\left|\theta(x) - \gamma_\theta(x; g)\right|^2P_X(dx)\right)^{1/2} \\
    &\leq \frac{1}{\alpha}\left(\int_{\calX}\E\left[\partial \ell(\theta, g; Z) \mid \theta(X) = \theta(x)\right]^2 P_X(dx)\right)^{1/2} \\
    &= \frac{1}{\alpha}\Cal(\theta, g).
    \end{align*}
    
\end{proof}

With the above two lemmas in hand, we can now prove Theorem~\ref{thm:conditional}, which we recall shows a decoupled bound on the calibration of a parameter $\theta$ with respect to a conditionally orthogonal loss function $\ell$.

\thmconditional*
\begin{proof}[Proof of Theorem~\ref{thm:conditional}]
    First, we note that by Corollary~\ref{cor:err_inv}, we have that
    \[
    \Cal(\theta, \eta_0) \equiv \Cal(\theta, (\eta_0, \zeta))
    \]
    for any second nuisance parameter $\zeta$, so it suffices to bound $\Cal(\theta, g_\theta)$ where $g_\theta = (\eta_0, \zeta_\theta)$.
    We have:
    \begin{align*}
    \E\left[\partial \ell(\theta, g_\theta; Z) \mid \theta(X)\right] &= \E\left[\partial \ell(\theta, g_\theta; Z) \mid \theta(X)\right] - \E\left[\partial \ell(\gamma_{\theta}^\ast, g_\theta; Z) \mid \theta(X)\right] \\
    & = \E\left[\partial^2 \ell(\wb{\theta}, g_\theta; Z) \mid \theta(X)\right](\theta(X) - \gamma_{\theta}^\ast(X)) \\
    & = \E\left[\partial^2 \ell(\wb{\theta}, g_\theta; Z) \mid \theta(X)\right](\theta(X) - \gamma_{\theta}(X; g)) \\
    & \quad + \E\left[\partial^2 \ell(\wb{\theta}, g_\theta; Z) \mid \theta(X)\right](\gamma_{\theta}(X; g) - \gamma_{\theta}^\ast(X)),
    \end{align*}
    where the first equality follows from the above calculation, the second from the fact $\gamma_\theta(x, (g_0, b)) = \gamma_\theta^\ast(x)$ regardless of choice of additional nuisance $b$, and the third from a first order Taylor expansion with Lagrange form remainder on $\theta(X)$ (here $\wb{\theta} \in [\theta, \gamma_{\theta}^\ast]$). The final equality follows from adding a subtracting $\gamma_{\theta}(X; g)$. Thus, we have
    \begin{align*}
    &\Cal(\theta, g_\theta)^2 = \E\left(\E\left[\partial \ell(\theta, g_\theta; Z) \mid \theta(X)\right]^2\right)\\
    &\qquad = \E\left(\E\left[\partial \ell(\theta, g_\theta; Z) \mid \theta(X)\right] \cdot \E\left[\partial^2 \ell(\wb{\theta}, g_\theta; Z) \mid \theta(X)\right]\cdot(\theta(X) - \gamma_{\theta}(X; g))\right) \\
    &\qquad \quad + \E\left(\E\left[\partial \ell(\theta, g_\theta; Z) \mid \theta(X)\right] \cdot \E\left[\partial^2 \ell(\wb{\theta}, g_\theta; Z) \mid \theta(X)\right]\cdot(\gamma_{\theta}(X; g) - \gamma_{\theta}^\ast(X))\right) \\
    &\qquad \leq \beta\Cal(\theta, g_\theta)\sqrt{\E\left[\left(\theta(X) - \gamma_{\theta}(X; g)\right)^2\right]} + \beta\Cal(\theta, g_\theta)\sqrt{\E\left[\left(\gamma_{\theta}(X; g) - \gamma_{\theta}^\ast(X)\right)^2\right]} \\
    &\qquad =\beta\Cal(\theta, g_\theta)\left\{\left\|\theta - \gamma_{\theta}(\cdot; g)\right\|_{L^2(P_X)} + \left\|\gamma_{\theta}(\cdot; g) - \gamma_{\theta}^\ast\right\|_{L^2(P_X)}\right\}.
    \end{align*}

    Now, dividing through by $\Cal(\theta, g_0)$ and plugging in the bounds provided by Lemma~\ref{lem:error_cal_cal} and Lemma~\ref{lem:error_est_cal} and again leveraging the equivalence $\Cal(\theta, \eta_0) = \Cal(\theta, g_\theta)$, we have
    \[
    \Cal(\theta, \eta_0) \leq \frac{\beta}{\alpha}\Cal(\theta, g) + \frac{\beta}{2\alpha}\err(g, g_\theta; \gamma_{\theta}^\ast),
    \]
    which is precisely the claimed result.
\end{proof}

\section{Algorithm Convergence Proofs}
\label{app:alg}

\subsection{Universal Orthogonality}
\label{app:alg:universal}
We now restate and prove the convergence guarantee of the sample splitting algorithm for calibration with respect to universally orthogonalizable loss functions. The result below is largely just an application of Theorem~\ref{thm:universal}, with the only caveat being that some care must be taken to handle the fact that the output parameter $\wh{\theta}$ and the nuisance estimate $\wt{g}$ are now random variables, not fixed parameters.
\thmalguniversal*
\begin{proof}[Proof of Theorem~\ref{thm:alg_universal}]
First, observe that, under Assumption~\ref{ass:gen_universal}, for any fixed $\theta$ and $g$, we have 
\begin{align*}
\underbrace{\Cal(\theta, g)}_{\text{causal cal. error}} = \underbrace{\Cal(\theta, P^\chi_g)}_{\text{non-causal cal. error}},
\end{align*}
where $P^\chi_g$ denotes the joint distribution of $(X, \chi(g; Z))$ over draws $Z \sim P_Z$.
and the latter quantity is controlled with high-probability by Assumption~\ref{ass:alg_univ}. Since the above equality holds for any $\theta$ and $g$, it still holds when $\theta$ is replaced by $\wh{\theta}$, the random output of Algorithm~\ref{alg:sample-split-universal}, and when $g$ is replaced by $\wh{g}$, the corresponding nuisance estimate obtained from $\calA_1$.

Under Assumption~\ref{ass:gen_universal}, it is also clear that $\err(g, g_0; \theta)$ does not depend on $\theta$, so it suffices to write $\err(g, g_0)$ going forward. Define the ``bad'' events as $B_1 := \{\err(\wh{g}, g_0) > \rate_1(n, \delta_1; P_Z)\}$ and $B_2 := \{\Cal(\wh{\theta}, \wh{P}^\chi_{\wh{g}}) > \rate_2(n, \delta_2; P^\chi_{\wh{g}})\}$. Clearly, the first part of Assumption~\ref{ass:alg_univ} yields that $\P(B_1) \leq \delta_1$. Likewise, the second part of Assumption~\ref{ass:alg_univ} also yields that $\P(B_2 \mid Z_{n + 1: 2n}) \leq \delta_1,$ as fixing $Z_{n + 1}, \dots, Z_{2n}$ fixes the learned nuisances $(\wh{g}, \wh{b})$, per Algorithm~\ref{alg:sample-split-universal}. Thus, applying the law of total probability, we have that the marginal probability of $B_2$ (over both draws of $Z_1, \dots, Z_n$ and $Z_{n + 1}, \dots, Z_{2n}$) is bounded by
\[
\P(B_2) = \E[\P(B_2 \mid Z_{1:n})] \leq \delta_2.
\]
This is because the $Z_i$ are independent and conditioning on the first $n$ observations fixes the nuisance estimate $\wh{g}$.
Thus, on the ``good'' event $B_1^c \cap B_2^c$, which occurs with probability at least $1 - \delta_1 - \delta_2$, we have
\begin{align*}
\Cal(\wh{\theta}, g_0) &\leq \frac{1}{2}\err(\wh{g}, g_0) + \Cal(\wh{\theta}, \wh{g})\\
&= \frac{1}{2}\err(\wh{g}, g_0) + \Cal(\wh{\theta}; P^{\chi}_{\wh{g}})\\
&\leq \frac{1}{2}\rate_1(n, \delta_1; P_Z) + \rate_2(n, \delta_2; P^\chi_{\wh{g}}),
\end{align*}
where the first inequality follows from Theorem~\ref{thm:universal} and the second equality follows from the preamble at the beginning of this proof.

\end{proof}

\subsection{Conditionally Orthogonality}

We now prove the convergence guarantees for Algorithm~\ref{alg:sample-split-cond}. The proof is largely the same as the convergence proof for Algorithm~\ref{alg:sample-split-universal}, but we nonetheless include it for completeness. The only key difference is that (a) we not longer have access to ``pseudo-outcomes'' that, in expectation, look like the target treatment effect and (b) we have to be careful, since the second nuisance parameter that is used in the definition of conditional orthogonality does not satisfy $\zeta_\theta = \zeta_{\wh{\theta}}$ in general. Once again, $\theta : \calX \rightarrow \R$ is the fixed, initial estimate treated as an input to Algorithm~\ref{alg:sample-split-cond} and $\wh{\theta}$ is the random, calibrated estimate that is the output of the algorithm.
\thmalgconditional*
\begin{proof}[Proof of Theorem~\ref{thm:alg_cond}]
First, we observe that, by the assumption that $\wh{\theta} = \wh{\tau}\circ \theta$ where the random map $\wh{\tau} : \R \rightarrow \R$ is almost surely injective, we know that initial estimate $\theta$ and the calibrated estimate $\wh{\theta}$ have the same level sets. Consequently, by Lemma~\ref{lem:level_set}, we have $\zeta_{\theta} = \zeta_{\wh{\theta}}$ without loss of generality. Further, we have equivalence of the calibration functions corresponding to $\theta$ and $\wh{\theta}$, i.e.\ we have $\gamma_\theta^\ast \equiv \gamma_{\wh{\theta}}^\ast$. Thus, letting $g_\varphi = (\eta_0, \zeta_\varphi)$ for any $\varphi : \calX \rightarrow \R$, we have 
\[
\err(\wh{g}, g_{\wh{\theta}}; \gamma_{\wh{\theta}}^\ast) = \err(\wh{g}, g_\theta; \gamma_\theta^\ast).
\]

The rest of the proof is now more or less identical to that of Theorem~\ref{thm:alg_universal}, but we nonetheless include the proof for completeness. Define the ``bad'' events as $B_1 := \{\err(\wh{g}, g_{\theta}; \gamma_\theta^\ast) > \rate_1(n, \delta_1, \theta; P_Z)\}$ and $B_2 := \{\Cal(\wh{\theta}; \wh{g}) > \rate_2(n, \delta_2, \ell; P_Z)\}$. Clearly, the first part of Assumption~\ref{ass:alg_cond} yields that $\P(B_1) \leq \delta_1$. Likewise, the second part of Assumption~\ref{ass:alg_cond} also yields that $\P(B_2 \mid Z_{n + 1: 2n}) \leq \delta_1,$ as fixing $Z_{n + 1}, \dots, Z_{2n}$ fixes the learned nuisances $(\wh{g}, \wh{b})$, per Algorithm~\ref{alg:sample-split-universal}. Thus, applying the law of total probability, we have that the marginal probability of $B_2$ (over both draws of $Z_1, \dots, Z_n$ and $Z_{n + 1}, \dots, Z_{2n}$) is bounded by
\[
\P(B_2) = \E[\P(B_2 \mid Z_{1:n})] \leq \delta_2.
\]
This is because the $Z_i$ are independent and conditioning on the first $n$ observations fixes the nuisance estimate $\wh{g}$.
Thus, on the ``good'' event $B_1^c \cap B_2^c$, which occurs with probability at least $1 - \delta_1 - \delta_2$, we have
\begin{align*}
\Cal(\wh{\theta}, \eta_0) &\leq \frac{1}{2}\err(\wh{g}, g_{\wh{\theta}}; \gamma_{\wh{\theta}}^\ast) + \Cal(\wh{\theta}, \wh{g})\\
&= \frac{1}{2}\err(\wh{g}, g_\theta; \gamma_\theta^\ast) + \Cal(\wh{\theta}; \wh{g})\\
&\leq \frac{1}{2}\rate_1(n, \delta_1, \theta; P_Z) + \rate_2(n, \delta_2, \ell; P_Z),
\end{align*}
where the first inequality follows from Theorem~\ref{thm:conditional} and the second equality follows from the preamble at the beginning of this proof.

\end{proof}

\section{A Do No Harm Property for Universally Orthogonal Losses}
\label{app:loss_red}

Throughout the main body of the paper, we focused on deriving bounds on excess calibration error for our causal calibration algorithm. One desideratum for any calibration algorithm is that the calibrated estimator does not have significantly higher loss than the original estimator. This property, called a ``do no harm'' property, is satisfied by many off-the-shelf calibration algorithms. We prove such a bound for universally orthogonal losses in the following section. In what follows, we denote the \emph{risk} of a certain predictor $\theta : \calX \rightarrow \R$ with respect to a nuisance $g : \calW \rightarrow \R^k$ as $R(\theta, g) := \E\ell(\theta, g; Z)$.

We start by proving a generic upper bound on the difference in risk between two arbitrary estimators. The argument presented below is similar in spirit to the proof of Theorem 2 in \citet{foster2023orthogonal}.

\begin{theorem}
\label{thm:risk_bd}
Let $\ell : \R \times \calG \times \calZ \rightarrow \R$ be a universally orthogonal loss function. Let $\theta, \theta' : \calX \rightarrow \R$ be two estimators, and $g \in \calG$ a nuisance function, and $g_0 \in \calG$ the true nuisance function. Then, we have
\[
R(\theta', g_0) - R(\theta, g_0) \leq R(\theta', g) - R(\theta, g) + \err'(g, g_0),
\]
where $\err'(g, g_0) := C \|g - g_0\|_{L^2(P_W)}^2$ if we assume $\ell$ is $C$-smooth in $g$, i.e. if we have
\begin{equation}
|D_g^2 R(\varphi, h)[g - g_0, g - g_0]| \leq C \|g - g_0\|_{L^2(P_W)}^2 \forall h \in [g, g_0], \varphi \in \Theta.
\end{equation}
If we instead assume $g = (\eta, \zeta)$ and $\ell(\theta, g; z) = \frac{1}{2}(\theta(x) - \chi(g; z))^2$ where
\[
\chi(g; z) := m(\eta; z) + \corr(g; z),
\]
$m(\eta; z)$ is linear in $\eta$, and $\E[\corr(g_0; Z) \mid X] = \E[\langle \zeta(W), (\eta_0 - \eta)(W)\rangle \mid X]$, then we take
\[
\err'(g, g_0) := \|\langle \eta - \eta_0, \zeta - \zeta_0\rangle\|_{L^2(P_W)}^2.
\]
\end{theorem}

\begin{proof}
First, observe that we have the bound
\[
R(\theta', g_0) - R(\theta, g_0) = \underbrace{R(\theta',  g_0) - R(\theta', g)}_{T_1} + \underbrace{R(\theta', g) - R(\theta, g)}_{T_2} + \underbrace{R(\theta, g) - R(\theta, g_0)}_{T_3}.
\]

Now, observe that by performing a second order Taylor expansion with respect to the nuisance component around $g_0$, we have;
\begin{align*}
T_1 &= -D_g R(\theta', g_0)(g - g_0) + \frac{1}{2}D_g^2 R(\theta', \wb{g})(g - g_0, g - g_0) \\
T_3 &= D_g R(\theta, g_0)(g - g_0) - \frac{1}{2}D_g^2 R(\theta, \wb{h})(g - g_0, g - g_0),
\end{align*}
where $\wb{g}, \wb{h} \in [g, g_0]$, understood in a point-wise sense. Thus, we can bound the sum:

\begin{align*}
T_1 + T_3 &= \underbrace{D_g R(\theta, g_0)(g - g_0) - D_g R(\theta', g_0)(g - g_0)}_{\text{first order difference}} \\
&\qquad + \underbrace{\frac{1}{2}D_g^2R(\theta', \wb{g})(g - g_0, g - g_0) - \frac{1}{2}D_g^2 R(\theta, \wb{h})(g - g_0, g - g_0)}_{\text{second order difference}}
\end{align*}

We can use universal orthogonality to show the first order difference vanishes. In particular, we have
\begin{align*}
D_g R(\theta, g_0)(g - g_0) - D_g R(\theta', g_0)(g - g_0) &= \E\left[D_g\E(\ell(\theta, g_0; Z) \mid X)(g - g_0) - D_g\E(\ell(\theta', g_0; Z) \mid X)(g - g_0)\right] \\
&= \E\left[D_g\E(\partial \ell(\wb{\theta}, g_0; Z) \mid X)(g - g_0)(\theta - \theta')(X)\right] \\
&=0
\end{align*}
where the second equality holds for some $\wb{\theta} \in [\theta, \theta']$\footnote{Understood in the sense that $\wb{\theta}(x) \in [\theta(x), \theta'(x)]$ for each $x \in \calX$} by performing a first-order Taylor expansion, and the final equality follows from universal orthogonality of $\ell$.

If we are in the first setting, and assume that the Hessian of the loss $\ell$ has a maximum eigenvalue uniformly bounded (over $\theta$ and $g$) by $\beta$, then we can bound the second order difference above as
\[
\frac{1}{2}D_g^2R(\theta, \wb{h})(g - g_0, g - g_0) - \frac{1}{2}D_g^2 R(\theta', \wb{g})(g - g_0, g - g_0) \leq \beta^2 \|g - g_0\|_{L^2(P_W)}^2.
\]
Else, if we instead assume the linear score condition outlined in Assumption~\ref{ass:gen_universal} and further assume $m(\eta, z)$ is linear in $\eta$, and thus have $\ell(\theta, g; z) = \frac{1}{2}(\theta(x) - \chi(g; z))^2$ for some pseudo-outcome mapping $\chi : \calG \times \calZ \rightarrow \R$, then it is clear we have for any $h \in \calG$
\begin{align*}
D_\eta^2\E[\ell(\theta, h; Z) \mid X](g - g_0, g - g_0) &= 0 \\
D_\zeta^2 \E[\ell(\theta, h; Z) \mid X](g - g_0, g - g_0) &= 0 \\
D_{\zeta,\eta} \E[\ell(\theta, h; Z) \mid X](g - g_0, g - g_0) &= -\left(\E[\langle (\zeta - \zeta_0)(W), (\eta - \eta_0)(W)\rangle \mid X]\right)^2.
\end{align*}

Thus, by leveraging Jensen's inequality, we can bound the second order difference by
\[
\leq 2\E[\langle (\eta - \eta_0)(W), (\zeta - \zeta_0)(W)\rangle^2] = \err'(g, g_0),
\]
which completes the proof.

\end{proof}

As before, with the above error decomposition bound, we can now prove a general convergence result given access to nuisance estimation and calibration algorithms that satisfy certain, high-probability convergence guarantees. Given the similarity of this result to Theorems~\ref{thm:alg_universal} and \ref{thm:alg_cond} and the fact Theorem~\ref{thm:risk_bd} does the majority of the heavy lifting , we omit the proof.

\begin{theorem}
Assume the same setup as Theorem~\ref{thm:risk_bd}. Assume $\calA_1 ; \calZ^\ast \rightarrow \calG$ is a nuisance estimation algorithm and $\calA_2 : \Theta \times (\calX \rightarrow \R)^\ast \rightarrow \Theta$ is a general loss calibration algorithm. Assume 
\begin{enumerate}
    \item For any distribution $P_Z$ on $\calZ$, $Z_1, \dots, Z_n \sim_{i.i.d} P_Z$, $\theta \in \Theta$, and failure probability $\delta_1 \in (0, 1)$, we have 
    \[
    \err(\wh{g}, g_0) \leq \rate_3(n, \delta_1, \theta; P_Z),
    \]
    where $\wh{g} \gets \calA_1(Z_{1:n})$ and $\rate_3$ is some rate function.
    \item For any distribution $P_Z$ on $\calZ$, $Z_1, \dots, Z_n \sim_{i.i.d.} P_Z$, $\theta \in \Theta$, $g \in \range(\calA_1)$, and failure probability $\delta_2 \in (0, 1)$, we have
    \[
    R(\wh{\theta}, \wh{g}) - R(\theta, \wh{g}) \leq \rate_4(n, \delta_2, \ell; P_z),
    \]
    where $\wh{\theta} \gets \calA_2(\theta, \{(X_m, \ell_m)\}_{m = 1}^n)$, $\ell_m := \ell(\cdot, \wh{g}; Z_m)$, and $\rate_4$ is some rate function.
\end{enumerate}
Then, with probability at least $1 - \delta_1 - \delta_2$, we have
\[
R(\wh{\theta}, g_0) - R(\theta, g_0) \leq \rate_3(n, \delta_1, \theta; P_Z) + \rate_4(n, \delta_2, \ell; P_Z).
\]
\end{theorem}
\section{Calibration Algorithms for General Losses}
\label{app:gen_cal}

Here, we present analogues of classical calibration algorithms defined with respect to general losses involving nuisance components. In particular, we give generalizations of histogram binning/uniform mass binning, isotonic calibration, and linear calibration. We only prove convergence guarantees for one algorithm (uniform mass binning), but we experimentally showed in Section~\ref{sec:experiments} that other algorithm (e.g.\ linear calibration) work well in practice. 

We first present a general algorithm that computes a univariate post-processing over a suitable class of functions $\calF \subset \{ f : \R \rightarrow \R\}$ and then returns a calibrated parameter estimate.

\begin{algorithm}[ht]
    \caption{General ERM-based Calibration}
    \label{alg:erm-cal}
\begin{algorithmic}[1]
   \Require{Samples $Z_1, \dots, Z_n \sim P_Z$, losses $\ell_1, \dots, \ell_n := \ell(\cdot, g_1; Z_1), \dots, \ell(\cdot, g_n; Z_n)$, estimator $\theta$, function class $\calF$.}
   \State{Compute $\wh{\tau} \in \arg\min_{\tau \in \calF}\sum_{m = 1}^n \ell_m(\tau\circ\theta(X_m))$.}
   \State{Define $\wh{\theta} := \wh{\tau} \circ \theta$.}
   \Ensure Calibrated estimator $\wh{\theta}$.
\end{algorithmic}
\end{algorithm}

Depending on how $\calF$ is defined, when we obtain generalizations of many classical algorithms. We refer the reader to Table~\ref{tab:cal_loss} in paper for appropriate choices of $\calF$. As noted in Section~\ref{sec:conditional}, our convergence guarantee required that we learn an injective post-processing map from model predictions to calibrated predictions. We discussed in detail how many common algorithms either naturally fit this desideratum or can be easily modified to fit this desideratum. The one exception to this is uniform mass/histogram binning, which maps model predictions to a finite number of values. To address this, we prove the convergence of a separate, three-way sample splitting algorithm that performs uniform mass binning for general losses.

\subsection{An analogue of Algorithm~\ref{alg:sample-split-cond} for UMB}
\label{sec:alg:cond}

We now develop a three-way sample splitting algorithm for causal calibration based on uniform mass/histogram binning \citep{gupta2020distribution, gupta2021distribution, kumar2019verified}. As in the case of Algorithm~\ref{alg:sample-split-cond}, our algorithm (Algorithm~\ref{alg:3way_umb} below) is implicitly based on the calibration error decomposition presented in Theorem~\ref{thm:conditional}. To prove the convergence of Algorithm~\ref{alg:sample-split-cond}, we needed to assume $\wh{\theta} = \wh{\tau} \circ \theta$ for some injective, randomized post-processing map $\wh{\tau} : \R \rightarrow \R$. While this desideratum is satisfied for some calibration algorithms (e.g.\ linear calibration and Platt scaling) and can be made to be satisfied for others (e.g.\ by slightly perturbing the post-processing map learned by isotonic calibration), it is clearly not a reasonable assumption for uniform mass binning/histogram binning. This is because, as the name suggest, binning algorithms ``compress'' the initial estimator into a pre-specified number of data-dependent bins.

Why did we assume $\wh{\theta}$ was an injective post-processing of $\theta$? In any causal calibration algorithm, we need to estimate the unknown nuisance functions. For conditionally orthogonal losses, there are two parameters we need etimate. The first parameter, $\eta_0$, does not depend on either the initial model $\theta(X)$ or the calibrated model $\wh{\theta}(X)$, and thus can be readily estimated from data. However, the second parameter, $\zeta_{\wh{\theta}}$, depends on the \textit{calibrated} model $\wh{\theta}$, which in general cannot be reliably estimated from data. Assuming the injectivity of the post-processing map $\wh{\tau}$ allowed us to conclude (via Lemma~\ref{lem:level_set}) that $\zeta_\theta = \zeta_{\wh{\theta}}$. Since we know $\theta$ at the outset of the problem, it is reasonable to assume we can produce a convergent estimate $\wh{\zeta}$ of $\zeta_{\wh{\theta}}$.

The question is now what can we do for binning-based algorithms, which generally cannot be written as an injective post-processing of the initial model. The key idea is that, for histogram binning/uniform mass binning, the level sets of $\wh{\theta}(X)$ only depend on the order statistics of $\theta(X_1), \dots, \theta(X_n)$, \textit{not} on the nuisance parameters. Thus, we propose a natural three-way sample-splitting analogue of uniform mass binning for conditionally orthogonal losses $\ell$.

Our algorithm works as follows. First, we use the covariates $X_1, \dots, X_n$ associated with one third of the data to compute the buckets/bins of $\wh{\theta}$ (these will be determined by evenly-spaced order statistics of $\theta(X_1), \dots, \theta(X_n)$). Fixing the level sets up front like this in turn fixes the additional nuisance we need to estimate. This is because, under Lemma~\ref{lem:level_set}, any function $\varphi : \calX \rightarrow \R$ with the same level sets as $\wh{\theta}$ will have $\zeta_\varphi = \zeta_{\wh{\theta}}$. Then, we use the second third $Z_{n + 1}, \dots, Z_{2n}$ to estimate the unknown nuisance functions $\eta_0$ and $\zeta_{\wh{\theta}}$. Finally, we use the learned nuisance parameters and the final third $Z_{2n + 1}, \dots, Z_{3n}$ to compute the empirical loss minimizer in each bucket.
We formally state this heuristically-described procedure below in Algorithm~\ref{alg:3way_umb}.

\begin{algorithm}[ht]
   \caption{Three-Way Uniform Mass Binning (UMB)}
   \label{alg:3way_umb}
\begin{algorithmic}[1]
   \Require Sample $Z_1, \dots, Z_{3n} \sim_{i.i.d.} P_Z$, loss $\ell(\nu, g; z)$, estimator $\theta$, nuisance est. alg.\ $\calA_1$, calibration alg.\ $\calA_2$, number of buckets $B$.
   \State Compute order statistics $\theta(X)_{(1)}, \dots, \theta(X)_{(n)}$ of $\theta(X_1), \dots, \theta(X_n)$.
   \State Set $\theta(X)_{(0)} = 0, \theta(X)_{(n)} = \infty$
   \For{$b \in [B]$}
    \State Set $V_b := [\theta(X)_{\lfloor (b - 1)N/B\rfloor}, \theta(X)_{\lfloor bN/B\rfloor})$
   \EndFor
   \State Define $\varphi(x) := \sum_{b = 1}^B \nu_b \mathbbm{1}[\theta(x) \in V_b]$ for any distinct $\nu_1, \dots, \nu_B$.
   \State Compute nuisances $\wh{g} = (\wh{\eta}, \wh{\zeta}) \gets \calA_1(\{Z_{m}\}_{m = n + 1}^{2n}, \varphi)$
   \For{$b \in [B]$}
    \State $\wh{\nu}_b := \arg\min_{\nu}\sum_{m = 2n + 1}^{3n} \ell(\nu, \wh{g}; Z_m)\mathbbm{1}[\theta(X_m) \in V_b]$
   \EndFor
   \State Set $\wh{\theta}(x) := \sum_{b = 1}^B \wh{\nu}_b \mathbbm{1}[\theta(x) \in V_b]$
   \Ensure Calibrated estimator $\wh{\theta}$.
\end{algorithmic}
\end{algorithm}

We now present the set of assumptions that we will use to prove convergence of the above algorithm.

\begin{ass}
\label{ass:alg_cond_bin}
Let $\calA_1: \Theta \times \calZ^\ast  \rightarrow \calG$ be a nuisance estimation algorithm and let  $\varphi: \calX \rightarrow \R$ be any initial estimator. 
We make the following assumptions.
\begin{enumerate}
\item $\range(\theta) \subset [0, 1]$.
\item The conditionally orthogonal loss function $\ell: \R \times \calG \times \calZ \rightarrow \R$ satisfies 
\begin{enumerate}
    \item For any $(\eta, \zeta) \in \calG, z \in \calZ$, the minimizer of the loss $\nu^\ast = \arg\min_{\nu} \ell (\nu, g; z)$ satisfies $\nu^\ast \in [0,1]$.
    \item For any $(\eta, \zeta) \in \calG, z \in \calZ$, and $\nu \in [0, 1]$, we have $\partial \ell(\nu, g; z) \in [-C, C]$.
\end{enumerate}
\item For any distribution $P_Z$ on $\calZ$, $n \geq 1$, and failure probability $\delta_1 \in (0, 1)$, with probability at least $1 - \delta_1$ over the draws of $Z_1, \dots, Z_n \sim P_Z$ i.i.d., we have
\[
\err(\wh{g}, g_\varphi; \gamma_{\varphi}^\ast) \leq \rate_1(n, \delta_1, \varphi; P_Z),
\]
where $\wh{g} := (\wh{\eta}, \wh{\zeta}) \gets \calA_1(\varphi, Z_{1:n})$ and $\rate_1$ is some rate function.
\end{enumerate}
\end{ass}

The above assumptions are mostly standard and essentially say the (a) the initial estimate $\theta(X)$ is bounded, (b) the partial derivative of the loss is bounded and the minimizer takes values in a bounded interval, and (c) we have access to some algorithm that can accurately estimate nuisance parameters.

For simplicity, we assume that the values $\wh{\nu}_1, \dots, \wh{\nu}_B$ Algorithm~\ref{alg:3way_umb} assigns to each of the buckets $V_1, \dots, V_B$ are unique. This is to ensure two distinct buckets $V_i \neq V_{j}$ do not merge, which would invalidate our application of Lemma~\ref{lem:level_set}. If, in practice, we have $\wh{\nu}_i = \wh{\nu}_{j}$ for $i \neq j$, the learner can simply add $\calU([-\epsilon, \epsilon])$ noise to $\wh{\nu}_i$ for $\epsilon > 0$ arbitrarily small to guarantee uniqueness. 

\begin{ass}
\label{ass:unique}
With probability 1 over the draws $Z_1, \dots, Z_{3n} \sim P_Z$ i.i.d., we have that $\wh{\nu}_1, \dots, \wh{\nu}_B \in [0, 1]$ are distinct. 

\end{ass}

We now state the main result of this subsection, a technical convergence guarantee for Algorithm~\ref{alg:3way_umb}. We prove Theorem~\ref{thm:3way_umb} (along with requisite lemmas and propositions) the next subsection of this appendix.

\begin{restatable}{theorem}{thmalgcond}
\label{thm:3way_umb}
Fix any initial estimator $\theta: \calX \rightarrow \R$, conditionally orthogonal loss function $\ell: [0, 1] \times \calG \times \calZ \rightarrow \R$, and failure probabilities $\delta_1, \delta_2 \in (0, 1)$. Suppose Assumptions~\ref{ass:strong_conv}, \ref{ass:smooth}, and \ref{ass:alg_cond_bin} hold, and assume $n \gtrsim B \log(B/\min(\delta_1, \delta_2))$. Then, with probability at least $1 - \delta_1 - \delta_2$ over the randomness of $Z_1, \dots, Z_{3n} \sim P_Z$, the output $\wh{\theta}$ of Algorithm~\ref{alg:3way_umb} satisfies
\begin{align*}
\Cal(\wh{\theta}, \eta_0) &\le \frac{\beta}{2\alpha} \rate_1(n, \delta_1, \wh{\theta}; P_Z)  + \frac{2\beta }{\alpha} \left(\frac{\beta}{n} + C \sqrt{\frac{2B\log(2nB/\delta_2)}{n}}.\right),
\end{align*}
where $C > 0$ is some constant that bounds the partial derivative as discussed in Assumption~\ref{ass:alg_cond_bin}: i.e. $|\partial \wt{\ell}(\nu, g; z)| < C$.
\end{restatable}

\subsection{Proving Theorem~\ref{thm:3way_umb}}

We now pivot to the task of proving Theorem~\ref{thm:3way_umb}. We first cite a result that guarantees, with high probability, that close to a $1/B$ fraction of points will fall into each bucket.

\begin{lemma}[\cite{kumar2019verified}, Lemma 4.3; \cite{gupta2020distribution}, Lemma 13]\label{lem:unif-bin-mass}
    For a universal constant $c > 0$, if $n \ge c B \log(B/\delta)$, the bucket $\{V_b : b \in [B]\}$ produced in Algorithm~\ref{alg:3way_umb} will satisfy
    \begin{equation}
    \label{eq:good_bucket}
        \frac{1}{2B} \le \P_{X \sim P_X}(\theta(X) \in V_b) \le \frac{2}{B},
    \end{equation}
    simultaneously for all $b \in [B]$ with probability at least $1-\delta$ over the randomness of $X_1, \dots, X_n$
\end{lemma}

Next, we argue that, conditional on the random bins satisfying Lemma~\ref{lem:unif-bin-mass}, the population average derivative conditional on the observation falling into any given bucket will be close to the corresponding sample average.

\begin{lemma}\label{lem:partial-deriv-concen-cond}
Fix any initial estimator $\theta: \calX \rightarrow \R$, conditionally orthogonal loss function $\ell: [0, 1] \times \calG \times \calZ \rightarrow \R$, and a nuisance estimate $g = (\eta, \zeta) \in \calG$. Suppose the buckets $\{V_b\}_{b \in [B]}$ are such that \[
    \frac{1}{2B} \le \P(\theta(X) \in V_i) \le \frac{2}{B}
\]
for every $b \in [B]$,  $n \ge 8 B \log(B/\delta)$, and Assumption~\ref{ass:alg_cond_bin} holds. Then, with probability at least $1-\delta$ over the randomness of $Z_{2n+1}, \dots, Z_{3n} \sim P^n_Z$, we have for all $b \in [B]$ and all $\nu \in [0,1]$
    \[
    \left|\E[ \partial \ell(\nu, g ; Z) \mid \theta(X) \in V_b] - \E_n[ \partial \ell(\nu, g ; Z) \mid \theta(X) \in V_b] \right| \le \frac{2\beta}{n} + 2C \sqrt{\frac{2B\log(\frac{4nB}{\delta})}{n}}
    \]
where \[\E_n[ \partial \ell(\nu, g ; Z) \mid \theta(X) \in V_b] := \frac{\sum_{m=2n+1}^{3n} \partial \ell(\nu, g ; Z_{m}) \cdot \mathbbm{1}[\theta(X_m) \in V_b]}{\sum_{m=2n+1}^{3n} \mathbbm{1}[\theta(X_m) \in V_b]} \] denotes the empirical conditional mean over the calibration dataset $Z_{2n+1}, \dots, Z_{3n}$.
\end{lemma}
\begin{proof}
    For convenience, let 
    \[
        S_b := \{ 2n+1 \le m \le 3m: \theta(X_m) \in V_b \}
    \]
    to denote set of indices that fall in $V_b$. Given that $n \ge 8 B \log(B/\delta)$,  the multiplicative Chernoff bound (Lemma~\ref{lem:multi-chernoff}) 
    tells us that with probability $1-\delta$,
    \begin{align}\label{eq:emp-measure-lb}
    |S_b| \ge \frac{n}{4B}.
    \end{align}
    for all $i \in [B]$.

    Note that $\E_n[ \partial \ell(\nu, g ; Z) \mid \theta(X) \in V_b]$
    is the empirical mean over $|S_b|$ many points. Therefore, with inequality~\eqref{eq:emp-measure-lb} and Hoeffding's inequality (Lemma~\ref{lem:hoeffding}), we have for any fixed $\nu \in [0,1]$, with probability at least $1 - \delta$, simultaneously for all $b \in [B]$,
    \begin{align*}
        \left|\E[ \partial \ell(\nu, g ; Z) \mid \theta(X) \in V_b] - \E_n[ \partial \ell(\nu, g ; Z) \mid \theta(X) \in V_b] \right| \le 2C \sqrt{\frac{\log(4B/\delta)}{2 |S_b|}} \le 2C \sqrt{\frac{2B\log(4B/\delta)}{n}}.
    \end{align*}
Now, for some $\epsilon >0$ to be chosen later (which we will implicitly assume satisfies $1/\epsilon \in \N$),  we now take a union bound over an $\epsilon$-covering of $[0,1]$: with probability $1-\delta$, we have for all  $\nu \in \{0, \epsilon, 2\epsilon, \dots, 1 - \epsilon\}$ and $b \in [B]$:
\begin{align*}
    \left|\E[ \partial \ell(\nu, g ; Z_m) \cdot \mid \theta(X) \in V_b] - \E_n[ \partial \ell(\nu, g ; Z_m) \mid \theta(X) \in V_b] \right| \le 2C \sqrt{\frac{2B\log(\frac{4B}{\epsilon\delta})}{n}}.
\end{align*}

For any $\nu \not\in \{0, \epsilon, 2\epsilon, \dots, 1 - \epsilon\}$, taking its closest point $\nu_\epsilon$ in the $\epsilon$-grid yields
\begin{align*}
    \left|\E[ \partial \ell(\nu, g ; Z) \cdot \mid \theta(X) \in V_b] - \E[ \partial \ell(\nu_\epsilon, g ; Z) \mid \theta(X) \in V_b] \right| \le \beta\epsilon \\
    \left|\E_n[ \partial \ell(\nu, g ; Z) \mid \theta(X) \in V_b] - \E_n[ \partial \ell(\nu_\epsilon, g ; Z) \mid \theta(X) \in V_b] \right| \le \beta \epsilon
\end{align*}
as $\partial \ell$ is $\beta$-Lipschitz by Assumption~\ref{ass:smooth}. Hence, with probability $1-\delta$, we have for any $b \in [B]$ and $\nu \in [0,1]$,
\begin{align*}
    &\left|\E[ \partial \ell(\nu, g ; Z) \mid \theta(X) \in V_b] - \E_n[ \partial \ell(\nu_b, g ; Z) \mid \theta(X) \in V_b] \right| \\
    &\le \left|\E[ \partial \ell(\nu, g ; Z) | \theta(X) \in V_b] - \E[ \partial \ell(\nu_\epsilon, g ; Z) \mid \theta(X) \in V_b]\right|\\
    &+\left|\E[ \partial \ell(\nu_\epsilon, g ; Z) \mid \theta(X) \in V_b] - \E_n[ \partial \ell(\nu_\epsilon, g ; Z) \mid \theta(X) \in V_b] \right|\\
    &+\left|\E_n[ \partial \ell(\nu, g ; Z) \mid \theta(X) \in V_b] - \E_n[ \partial \ell(\nu_\epsilon, g ; Z) \mid \theta(X) \in V_b] \right|\\
    &\le 2\beta \epsilon + 2C \sqrt{\frac{\log(2B/\epsilon \delta)}{2n}}.
\end{align*}
    Therefore, we have with probability $1-\delta$,
    \[
        \left|\E[ \partial \ell(\nu, g ; Z) \mid \theta(X) \in V_b] - \E_n[ \partial \ell(\nu, g ; Z) \mid \theta(X) \in V_b] \right| \le 2\beta\epsilon + 2C \sqrt{\frac{2B\log(\frac{4B}{\epsilon\delta})}{n}}.
    \]
    for all $b \in [B]$ and $\nu \in [0,1]$.
    Setting $\epsilon = \frac{1}{n}$ yields the result
\end{proof}

\begin{lemma}[Multiplicative Chernoff Boud]\label{lem:multi-chernoff}
Let $X_1, \dots, X_n$ be independent random variables such that $X_i \in \{0,1\}$ and $\E[X_i]=p$ for all $i \in [n]$. For all $t \in (0,1)$,\[
\P\left(\sum_{i=1}^n X_i \le (1-t)np\right)\le \exp\left(-\frac{npt^2}{3}\right).
\]
    
\end{lemma}
\begin{lemma}[Hoeffding's Inequality]
\label{lem:hoeffding}
Let $X_1, \dots, X_n$ be independent random variables such that $X_i \in [a,b]$. Consider the sum of these random variables $S_n = X_1 + \dots + X_n$. For all $t > 0$, 
\[
\P\left(|S_n - \E[S_n]| > t\right) \le 2\exp\left(\frac{2t^2}{n (b-a)^2}\right).
\]
\end{lemma}

We use Lemma~\ref{lem:partial-deriv-concen-cond} to now prove the following Proposition.

\begin{prop}
\label{prop:bucket}
Assume the same setup as Lemma~\ref{lem:partial-deriv-concen-cond} and suppose Assumption~\ref{ass:alg_cond_bin} holds. Let $Z_1, \dots, Z_n \sim P_Z$ be i.i.d., let $g = (\eta, \zeta) \in \calG$ be a fixed nuisance pair, and set $\wh{\theta} := \sum_{b = 1}^B\wh{\nu}_b \mathbbm{1}[\theta(x) \in V_b]$, where \[\wh{\nu}_b = \arg\min_\nu\sum_{m = 1}^n \ell(\nu, g; Z_m)\mathbbm{1}[\theta(X_m) \in V_b].\] Assume the $\wh{\nu}_b$ are distinct almost surely. Then, for any $\delta \in (0, 1)$, we have with probability at least $1 - \delta$, 
\[
\Cal(\wh{\theta}, g) \le \frac{2\beta}{n} + 2C \sqrt{\frac{2B\log(4nB/\delta)}{n}}
\]
where $C > 0$ is some constant that bounds the partial derivative as discussed in Assumption~\ref{ass:alg_cond_bin}: i.e. $|\partial \ell(\nu, g; z)| < C$.
\end{prop}

\begin{proof}

As we have assumed that all $\wh{\nu}_b$'s are all distinct, we have
\begin{align*}
\Cal(\wh{\theta}, g) &= \sqrt{\sum_{b \in [B]} \P(\wh{\theta}(X) = \wh{\nu}_b) \cdot \E\left[\partial \ell(\wh{\nu}_b, g; Z) \mid \wh{\theta}(X) = \nu_b\right]^2}
\end{align*}

Since we have assumed $n \geq 8B \log(B/\delta)$, we have from Lemma~\ref{lem:partial-deriv-concen-cond} that, with probability at least $1-\delta$, simultaneously for each $b \in [B]$
\begin{align*}
    \E[\partial \ell(\wh{\nu}_b, g; Z) \mid \wh{\theta}(X) = \wh{\nu}_b]  = \E[\partial \ell(\wh{\nu}_b, g; Z) \mid \theta(X) \in V_b] \leq \frac{2\beta}{n} + 2C \sqrt{\frac{2B\log(\frac{4nB}{\delta})}{n}},
\end{align*}
which follows since $\E_n[\partial \ell(\wh{\nu}_b, g; Z) \mid \theta(X) \in V_b] = 0$ by definition of $\wh{\nu}_b$. Thus, with probability at least $1 - \delta$, we get 
\begin{align*}
    \Cal(\wh{\theta}, g) &= \sqrt{\sum_{b \in [B]} \P(\theta(X) \in V_b) \cdot \E\left[\partial \ell(\wh{\nu}_b, g; Z) \mid \wh{\theta}(X) = \wh{\nu}_b\right]^2}\\
    &\le \sqrt{\sum_{b \in [B]}\P(\theta(X) \in V_b) \cdot \left(\frac{2\beta}{n} + 2C \sqrt{\frac{2B\log(\frac{4nB}{\delta})}{n}}\right)^2}\\
    &=\frac{2\beta}{n} + 2C \sqrt{\frac{2B\log(\frac{4nB}{\delta})}{n}}.
\end{align*}
\end{proof}

We now have the requisite tools to prove Theorem~\ref{thm:3way_umb}. Our argument proceeds in largely the same way that the proof of Theorems~\ref{thm:alg_universal} and \ref{thm:alg_cond} --- we start by defining appropriate ``good'' events, and then subsequently bound the overall probability of their failure.

\begin{proof}[Proof of Theorem~\ref{thm:3way_umb}]
As before we start by defining some ``bad'' events. In particular, consider the events $B_1$ and $B_2$ defined respectively by
\[
B_1 := \left\{\err(\wh{g}, g_{\wh{\theta}}; \gamma_{\wh{\theta}}^\ast) > \rate_1(n, \delta_1, \wh{\theta}; P_Z)\right\} \quad \text{and} \quad B_2 := \left\{\Cal(\wh{\theta}, \wh{g}) > \frac{2\beta}{n} + 2C\sqrt{\frac{2B\log(8nB/\delta_2)}{n}}\right\}
\]
where $\wh{g} = (\wh{\eta}, \wh{\zeta}) \gets \calA_1(\varphi, Z_{n+1:2n})$, where $\varphi$ is as defined in Algorithm~\ref{alg:3way_umb}. It is clear that $\P(B_1 \mid Z_{1:n}) \leq \delta_1$ by Assumption~\ref{ass:alg_cond_bin}, since fixing $Z_{1:n}$ fixes the estimate $\varphi$. Thus, the law of total expectation yields that $\P(B_1) \leq \delta_1$. 

Bounding $\P(B_2)$ takes mildly more care. Define the event 
\[
E := \left\{\forall b \in [B], \frac{1}{2B} \leq \P(\theta(X) \in V_b) \leq \frac{2}{B}\right\},
\]
which by Lemma~\ref{lem:unif-bin-mass} occurs with probability at least $1 - \delta_2/2$ by the assumption that $n \geq cB\log(2B/\delta_2)$. We have the following bound:
\begin{align*}
\P(B_2) &= \P(B_2 \mid E)\P(E) + \P(B_2 \mid E^c)\P(E^c) \\
&\leq \P(B_2 \mid E) + \P(E^c) \\
&= \E\left[\P(B_2 \mid E, Z_{1:2n}) \mid E\right] + \P(E^c) \\
&\leq \delta_2/2 + \delta_2/2 = \delta_2,
\end{align*}
where the second to last inequality follows form the fact that $\P(B_2 \mid E, Z_{1:2n}) \leq \delta_2$ by Proposition~\ref{prop:bucket}. We now apply Theorem~\ref{thm:conditional} to see that, on the ``good'' event $G = B_1^c \cap B_2^c$ (which occurs with probability at least $1 - \delta_1 - \delta_2$) we have

    \begin{align*}
        \Cal(\wh{\theta}, \eta_0) &\leq \frac{\beta}{2\alpha}\err(\wh{g}, g_{\wh{\theta}}; \gamma_{\wh{\theta}}^\ast) + \frac{\beta}{\alpha}\Cal(\wh{\theta}, \wh{g})  \\
        &\leq \frac{\beta}{2\alpha} \rate_1(n, \delta_1, \wh{\theta}; P_Z) + \frac{2\beta }{\alpha} \left(\frac{\beta}{n} + C \sqrt{\frac{2B\log(8nB/\delta_2)}{n}}\right),
    \end{align*}
    which proves the desired result.
\end{proof}

\section{Additional Proofs}
\label{app:add}
In this appendix, we proofs of additional claims that do not constitute our primary results. We proceed by section of the paper.
\propcrosserror*
\begin{proof}[Proof of Proposition~\ref{prop:cross_error}]
Let $\wb{g} = (\wb{\eta}, \wb{\zeta}) \in [g, g_0]$. Observe that we have,
\begin{align*}
&D_{g}^2\E[\partial \ell(\theta, \wb{g}; Z) \mid X](g - g_0, g - g_0)\\
&= D_{g}^2\E[\theta(X) - m(\wb{\eta}; Z) \mid X](g - g_0, g - g_0) \\
&\qquad - D_{g}^2\E[\langle \wb{\zeta}(W) , \wb{\eta}(W) - \eta_0(W)) \mid X](g - g_0, g - g_0) &\text{(Assumption~\ref{ass:gen_universal})} \\
&= D_{g}^2\E[\langle \wb{\zeta}(W),(\wb{\eta} - \eta_0)(W)\rangle\mid X](g - g_0, g - g_0) &\text{(Linearity of $m(\eta; z)$)}.
\end{align*}

Now, we compute the Hessian that appears on the second line. In particular, further calculation yields that
\begin{align*}
&D_{g}^2\E[\langle \wb{\zeta}(W), \wb{\eta}(W) - \eta_0(W)\rangle\mid X](g - g_0, g - g_0)\\
&\qquad = \E\left[\left(\eta(W) - \eta_0(W), \zeta(W) - \zeta_0(W)\right)^\top\begin{pmatrix} 0 & - 1 \\ -1 & 0\end{pmatrix}\left(\eta(W) - \eta_0(W), \zeta(W) - \zeta_0(W)\right) \mid X \right] 
\end{align*}
Thus, we can write the error term down as 
\begin{align*}
&\err(g, g_0; \theta)^2 := \sup_{f \in [g, g_0]}\E\left(\left\{D_{g}^2\E\left[\partial \ell(\theta, f; Z) \mid X\right](g - g_0, g - g_0)\right\}^2\right) \\
&\qquad = \E\left( \E\left[((\eta - \eta_0)(W), (\zeta - \zeta_0)(W))^\top\begin{pmatrix} 0 & - 1 \\ -1 & 0\end{pmatrix}((\eta - \eta_0)(W), (\zeta - \zeta_0)(W)) \mid X \right]^2\right) \\
&\qquad = 4\E\left(\E\left[\langle\eta - \eta_0)(W), \zeta - \zeta_0)(W)
\rangle\mid X\right]^2\right) \\
&\qquad \leq 4\|(g - g_0)(b - b_\theta)\|_{L^2(P_W)}^2 \qquad\qquad\qquad \text{(Jensen's Inequality and Tower Rule)}.
\end{align*}
Thus, taking square roots, we see that we have
\[
\err(g, g_0) \leq 2\|\langle \eta - \eta_0, \zeta - \zeta_0\rangle\|_{L^2(P_W)}.
\]
Dividing both sides by two yields the claimed result. The second claim is trivial and follows from an application of Jensen's inequality.

\end{proof}
{\renewcommand\footnote[1]{}\lemlevelset*}
\begin{proof}[Proof of Lemma~\ref{lem:level_set}]
The first claim is straightforward, as we have 
\begin{align*}
\E[\ell(\nu, g; Z) \mid \varphi_1(X)] = \E[\ell(\nu, g; Z) \mid \varphi_2(X)]
\end{align*}
for all $\nu \in \R$, $g = (\eta, \zeta) \in \calG$, so we just plug in $g_0 = (\eta_0, \zeta)$ for any $\zeta$. The equivalence of calibration functions, in particular, implies we have
\[
0 = D_g\E[\partial \ell(\gamma_{\varphi_{c_2}}^\ast, (\eta, \zeta_{\varphi_{c_1}}) ; Z) \mid \varphi_{c_1}(X)](g - g_0) = D_g\E[\partial \ell(\gamma_{\varphi_{c_2}}^\ast, (\eta_0, \zeta_{\varphi_{c_1}}); Z) \mid \varphi_{c_2}(X)](g - g_0)
\]
for any direction $g \in \calG$ and choice of $c_1, c_2 \in \{1, 2\}$. That is, $\ell$ satisfies the second point of Definition~\ref{def:conditional} when either $\zeta_{\varphi_1}$ or $\zeta_{\varphi_2}$ are plugged in. Thus, without loss of generality, we can assume they are the same. 
\end{proof}

\end{document}